%% file: primal_dual.tex
\documentclass[11pt]{article}
\usepackage{fullpage}
\usepackage{dsfont}
\usepackage{smile}
\usepackage[colorlinks=true,
linkcolor=blue,
urlcolor=blue,
citecolor=blue]{hyperref}
\usepackage[overload]{empheq}
\usepackage{wrapfig}
\usepackage{kpfonts}
\DeclareMathAlphabet{\mathsf}{OT1}{cmss}{m}{n}
\SetMathAlphabet{\mathsf}{bold}{OT1}{cmss}{bx}{n}

\usepackage{enumitem}
\newcommand{\sym}{\mathop{\mathrm{sym}}}

\newcommand{\rk}{\mathop{\mathrm{rank}}}

\newcommand{\vect}{\mathop{\mathrm{vec}}}

\providecommand{\norm}[1]{\|#1\|}

\newcommand{\removed}[1]{}

\newcommand*{\Ic}{\cI^{\perp}}
\usepackage[left]{lineno}
\usepackage{blindtext}

\newcommand*\patchAmsMathEnvironmentForLineno[1]{%
  \expandafter\let\csname old#1\expandafter\endcsname\csname #1\endcsname
  \expandafter\let\csname oldend#1\expandafter\endcsname\csname end#1\endcsname
  \renewenvironment{#1}%
     {\linenomath\csname old#1\endcsname}%
     {\csname oldend#1\endcsname\endlinenomath}}% 
\newcommand*\patchBothAmsMathEnvironmentsForLineno[1]{%
  \patchAmsMathEnvironmentForLineno{#1}%
  \patchAmsMathEnvironmentForLineno{#1*}}%
\AtBeginDocument{%
\patchBothAmsMathEnvironmentsForLineno{equation}%
\patchBothAmsMathEnvironmentsForLineno{align}%
\patchBothAmsMathEnvironmentsForLineno{flalign}%
\patchBothAmsMathEnvironmentsForLineno{alignat}%
\patchBothAmsMathEnvironmentsForLineno{gather}%
\patchBothAmsMathEnvironmentsForLineno{multline}%
}

\title{\huge \bf On Landscape of Lagrangian Function\\ and Stochastic Search for\\ Constrained Nonconvex Optimization\footnote{Working in Progress}}

\date{}

\author{Zhehui Chen, Xingguo Li, Lin F. Yang, Jarvis Haupt, and Tuo Zhao\thanks{Zhehui Chen and Xingguo Li contribute equally; Zhehui Chen and Tuo Zhao are affiliated with School of Industrial and Systems Engineering at Georgia Institute of Technology; Xingguo Li and Jarvis Haupt are affiliated with Department of Electrical and Computer Engineering at University of Minnesota; Lin F. Yang is affiliated with Operation Research and Financial Engineering Department at Princeton University; Email:{tourzhao@gatech.edu}; Tuo Zhao is the corresponding author.}}

\begin{document}

\maketitle

\begin{abstract}
We study constrained nonconvex optimization problems in machine learning, signal processing, and stochastic control. It is well-known that these problems can be rewritten to a minimax problem in a Lagrangian form. However, due to the lack of convexity, their landscape is not well understood and how to find the stable equilibria of the Lagrangian function is still unknown. To bridge the gap, we study the landscape of the Lagrangian function. Further, we define a special class of Lagrangian functions. They enjoy two properties: 1.Equilibria are either stable or unstable (Formal definition in Section~\ref{sec-equi}); 2.Stable equilibria correspond to the global optima of the original problem. We show that a generalized eigenvalue (GEV) problem, including canonical correlation analysis and other problems, belongs to the class. Specifically, we characterize its stable and unstable equilibria by leveraging an invariant group and symmetric property (more details in Section~\ref{GEV}). Motivated by these neat geometric structures, we propose a simple, efficient, and stochastic primal-dual algorithm solving the online GEV problem. Theoretically, we provide sufficient conditions, based on which we establish an asymptotic convergence rate and obtain the first sample complexity result for the online GEV problem by diffusion approximations, which are widely used in applied probability and stochastic control. Numerical results are provided to support our theory.
\end{abstract}

\input{intro}
\input{landscape}

\input{algorithm}

\input{Discussion}
\bibliographystyle{ims}
\bibliography{ref.bib}

\appendix
\input{Append-1}
\input{Append-2}

\end{document}

%% file: intro.tex
%!TEX root = ./primal_dual.tex

\section{Introduction}

We often encounter the following optimization problem in machine learning, signal processing, and stochastic control:
\begin{align}\label{eqn-obj-1}
\min_{X}f(X)\quad\textrm{subject to}\quad X\in \Omega,
\end{align}
where $f:\RR^d\rightarrow \RR$ is a loss function, $\Omega:\triangleq \{X\in \RR^d:g_i(X)= 0, i=1,2,...,m\}$ denotes a feasible set, $m$ is the number of constraints, and $g_i:\RR^{d}\rightarrow \RR$'s are the differentiable functions that impose constraints into model parameters. For notational simplicity, we define $\cG(X)=[g_1(X),...,g_m(X)]^\top$ and $\Omega= \{ X\in\RR^d: \cG(X)=0\}$. Principal component analysis (PCA), canonical correlation analysis (CCA), matrix factorization/sensing/completion, phase retrieval, and many other problems \citep{friedman2001elements,sun2016geometric,bhojanapalli2016global,li2016symmetry,ge2016matrix,chen2017online,zhu2017global} can be viewed as special examples of \eqref{eqn-obj-1}. Many algorithms have been proposed to solve \eqref{eqn-obj-1}. For the unconstrained ($\Omega=\RR^d$) or a simple constraint $\cG(X)$, e.g., the spherical constraint, $\cG(X):=\norm{X}_2-1$, we can apply simple first order algorithms such as the projected gradient descent algorithm \citep{luenberger1984linear}.

However, when $\cG(X)$ is complicated, the aforementioned algorithms are often not applicable or inefficient. This is because the projection to $\Omega$ does not admit a closed form expression and can be computationally expensive in each iteration. To address this issue, we convert \eqref{eqn-obj-1} to a min-max problem using the Lagrangian multiplier method. Specifically, instead of solving \eqref{eqn-obj-1}, we solve the following problem:
\begin{align}\label{eqn-Lag-1}
\min_{X\in \RR^d} \max_{Y \in \RR^m}~ \cL(X,Y):= f(X)+Y^\top \cG(X),
\end{align}
where $Y\in\RR^m$ is the Lagrangian multiplier. $\cL(X,Y)$ is often referred as the Lagrangian function in existing literature \citep{boyd2004convex}. The existing literature on optimization also refers to $X$ as the primal variable and $Y$ as the dual variable. Accordingly, \eqref{eqn-obj-1} is called the primal problem. From the perspective of game theory, they can be viewed as two players competing with each other and eventually achieving some equilibrium. When $f(X)$ is convex and $\Omega$ is convex or the boundary of a convex set, the optimization landscape of \eqref{eqn-Lag-1} is essentially convex-concave, that is, for any fixed $Y$, $\cL(X,Y)$ is convex in $X$, and for any fixed $X$, $\cL(X,Y)$ is concave in $Y$. Such a landscape further implies that the equilibrium of \eqref{eqn-Lag-1} is a saddle point, whose primal variable is equivalent to the global optimum of \eqref{eqn-obj-1} under strong duality conditions. To solve \eqref{eqn-Lag-1}, we resort to primal-dual algorithms, which iterate over both $X$ and $Y$ (usually in an alternating manner). The global convergence rates to the equilibrium are also established accordingly for these algorithms \citep{lan2011primal,chen2014optimal,iouditski2014primal}. 

When $f(X)$ and $\Omega$ are nonconvex, both \eqref{eqn-obj-1} and \eqref{eqn-Lag-1} become much more computationally challenging, NP-Hard in general. Significant progress has been made toward solving the primal problem \eqref{eqn-obj-1}. For example, \cite{ge2015escaping} show that when certain tensor factorization satisfies the so-called strict saddle properties, one can apply some first order algorithms such as the projected gradient algorithm, and the global convergence in polynomial time can be guaranteed. Their results further motivate many follow-up works, proving that many problems can be formulated as strict saddle optimization problems, including PCA, multiview learning, phase retrieval, matrix factorization/sensing/completion, complete dictionary learning \citep{sun2016geometric,bhojanapalli2016global,li2016symmetry,ge2016matrix,chen2017online,zhu2017global}. Note that these strict saddle optimization problems are either unconstrained or just with a simple spherical constraint. However, for many other nonconvex optimization problems, $\Omega$ can be much more complicated. To the best of our knowledge, when $\Omega$ is not only nonconvex but also complicated, the applicable algorithms and convergence guarantees are still largely unknown in existing literature.

To handle the complicated $\Omega$, this paper proposes to investigate the min-max problem \eqref{eqn-Lag-1}. Specifically, we first define a special class of Lagrangian functions, where the landscape of  $\cL(X,Y)$ enjoys the following good properties:
\begin{itemize}[leftmargin=*]
\item {\it There exist only two types of equilibria -- stable and unstable equilibria. At an unstable equilibrium, $\cL(X,Y)$ has negative curvature with respect to the primal variable $X$. More details in Section~\ref{sec-equi}.}
\item {\it All stable equilibria correspond to the global optima of the primal problem \eqref{eqn-obj-1}}.
\end{itemize}
Both properties are intuitive. On the one hand, the negative curvature in the first property enables the primal variable to escape from the unstable equilibria along some decent direction. %; while with the restricted strong convexity enforces the primal variable to stay at the stable equilibria. 
On the other hand, the second property ensures that we do not get spurious local optima of \eqref{eqn-obj-1}, that is all local minima must also be global optima.

We then study a generalized eigenvalue (GEV) problem, which includes CCA, Fisher discriminant analysis (FDA, \cite{mika1999fisher}),  sufficient dimension reduction (SDR, \cite{cook2005sufficient}) as special examples. Specifically, GEV solves  
 \begin{align}
	X^*=\argmin_{X \in \RR^{d\times r}}~f(X):={-\tr(X^\top A X)}~~\textrm{s.t.}~~X \in\cT_{B}:=  \{ X\in \RR^{d\times r} : X^\top B X=I_r \}, \label{eqn:original}
\end{align}
where $A,B \in \RR^{d \times d}$ are symmetric, $B$ is positive semidefinite. We rewrite \eqref{eqn:original} as a min-max problem,
\begin{align}
\min_X\max_Y\cL(X,Y)=-\tr( X^\top A X) + \langle Y,X^\top B X-I_r \rangle,\label{eqn:minimax}
\end{align}
 where $Y\in \RR^{r\times r}$ is the Lagrangian multiplier. Theoretically, we show that the Lagrangian function in \eqref{eqn:minimax} exactly belongs to our previously defined class. Motivated by our defined landscape structures, we then solve an online version of \eqref{eqn:minimax}, where we can only access independent unbiased stochastic approximations of $A,~B$ and directly accessing $A$ and $B$ is prohibited. Specifically, at the $k$-th iteration, we only obtain independent $A^{(k)}$ and $B^{(k)}$ satisfying
 \begin{align*}
 \EE A^{(k)} =A\quad\textrm{and}\quad\EE B^{(k)} = B.
 \end{align*}
Computationally, we propose a simple stochastic primal-dual algorithm, which is a stochastic variant of the generalized Hebbian algorithm (GHA, \cite{gorrell2006generalized}). Theoretically, we establish its asymptotic rate of convergence to stable equilibria for our stochastic GHA (SGHA) based on the diffusion approximations~\citep{kushner2003stochastic}. Specifically, we show that, asymptotically, the solution trajectory of SGHA weakly converges to the solutions of stochastic differential equations (SDEs). By studying the analytical solutions of these SDEs, we further establish the asymptotic sample/iteration complexity of SGHA under certain regularity conditions~\citep{harold1997stochastic,li2016near,chen2017online}. To the best of our knowledge, this is the first asymptotic sample/iteration complexity analysis of a stochastic optimization algorithm for solving the online version of GEV problem. Numerical experiments are presented to justify our theory.

Our work is closely related to several recent results on solving GEV problems. For example, \cite{ge2016efficient} propose a multistage semi-stochastic optimization algorithm for solving GEV problems with a finite sum structure. At each optimization stage, their algorithm needs to access the exact $B$ matrix, and compute the approximate inverse of $B$ by solving a quadratic program, which is not allowed in our setting. Similar matrix inversion approaches are also adopted by a few other recently proposed algorithms for solving GEV problem \citep{allen2016doubly,arora2017stochastic}. In contrast, our proposed SGHA is a fully stochastic algorithm, which does not require any matrix inversion.  

Moreover, our work is also related to several more complicated min-max problems, such as Markov Decision Process with function approximation, Generative Adversarial Network, multistage stochastic programming and control~\citep{sutton2000policy,shapiro2009lectures,goodfellow2014generative}. Many primal-dual algorithms have been proposed to solve these problems. However, most of these algorithms are even not guaranteed to converge. As mentioned earlier, when the convex-concave structure is missing, the min-max problems go far beyond the existing theories. Moreover, both primal and dual iterations involve sophisticated stochastic approximations (equally or more difficult than our online version of GEV). This paper makes the attempt on understanding the optimization landscape of these challenging min-max problems. Taking our results as an initial start, we expect more sophisticated and stronger follow-up works that apply to these min-max problems.

\noindent \textbf{Notations}.  Given an integer $d$, we denote  $I_d$ as a $d \times d$ identity matrix, $[d] = \{1,2,\ldots,d \}$. Given an index set $\cI\subseteq [d]$ and %$\cS \subseteq [d]$, we denote $\Sc = [d] \backslash \cS$ as the complement set of $\cS$. 
 a matrix $X \in \RR^{d \times r}$, we denote $\Ic = [d] \backslash \cI$ as the complement set of $\cI$, $X_{:,i}$ ($X_{i,:}$) as the $i$-th column (row) of $X$, $X_{i,j}$ as the $(i,j)$-th entry of $X$, and $X_{:,\cI}$ ($X_{\cI,:}$) as the column (row) submatrix of $X$ indexed by $\cI$, $\vect(X) \in \RR^{dr}$ as the vectorization of $X$, $\textrm{Col}(X)$ as the column space of $X$, and $\textrm{Null}(X)$ as the null space of $X$. Given a symmetric matrix $X \in \RR^{d \times d}$, we denote $\lambda_{\min\slash\max}(X)$ as its smallest\slash largest singular value, and denote the eigenvalue decomposition of $X$ as $X = O \Lambda O^{\top}$, where $\Lambda=\diag(\lambda_1,...\lambda_d)$ with $\lambda_1\geq\cdots\geq\lambda_{d}$,  denote $\norm{X}_2$ as the spectral norm of $X$. Given two matrices $X$ and $Y$, $X \otimes Y$ as the Kronecker product of $X$, $Y$.

%% file: landscape.tex
%!TEX root = ./primal_dual.tex

\section{Characterization of Equilibria}\label{sec-equi}

Recall the Lagrangian function in \eqref{eqn-Lag-1}. Then we start with characterizing its equilibria. By KKT conditions, an equilibrium $(X,Y)$  satisfies
\begin{align*}
\nabla_{X} \cL(X,Y) = \nabla_X f(X) + Y^\top \nabla_X \cG(X) = 0\quad\textrm{and}\quad\nabla_{Y} \cL(X,Y) = \cG(X) = 0,
\end{align*}
which only contains the first order information of $\cL(X,Y)$. To further distinguish the difference among the equilibria, we define two types of equilibria by the second order information.
\begin{definition}\label{def:crit-point}
Given the Lagrangian function $\cL(X,Y)$ in \eqref{eqn-Lag-1}, a point $(X,Y)$ is called:
\begin{itemize}[leftmargin=*]
\item {\bf (1)} An \textbf{\textit{equilibrium}} of $\cL(X,Y)$, if
\begin{align*}
\displaystyle\nabla \cL(X,Y)= \left[ \begin{array}{c}
		\nabla_X \cL(X,Y) \\
		\nabla_Y \cL(X,Y)
	\end{array} \right] = 0.
\end{align*}

\item {\bf (2)}   An \textbf{\textit{equilibrium}} $(X,Y)$ is \textbf{\textit{unstable}},  if $(X,Y)$ is an equilibrium and $\lambda_{\min} \left(\nabla_X^2 \cL(X,Y) \right) < 0.$

\item {\bf (3)} An \textbf{\textit{equilibrium}} $(X,Y)$ is \textbf{\textit{stable}}, if $(X,Y)$ is an equilibrium, $\nabla_X^2 \cL(X,Y) \succeq 0$, and $\cL(X,Y)$ is strongly convex over a restricted domain.%domain (satisfying Polyak \L ojasiewicz conditions).
\end{itemize}
\end{definition}

Note that (2) in~Definition~\ref{def:crit-point} has a similar strict saddle property over a manifold in~\cite{ge2015escaping}. The motivation behind Definition~\ref{def:crit-point} is intuitive. When $\cL(X,Y)$ has negative curvature with respect to the primal variable $X$ at an equilibrium, we can find a direction in $X$ to further decrease $\cL(X,Y)$. Therefore, a tiny perturbation can break this unstable equilibrium. An illustrative example is presented in Figure~\ref{illustrate}. Moreover, at a stable equilibrium $(X^*,Y^*)$, there is restricted strong convexity, which relates to several conditions, e.g., Polyak \L ojasiewicz conditions \citep{polyak1963gradient}, i.e., $$\norm{\nabla_X\cL(X,Y^*)}^2\geq \mu(\cL(X,Y^*)-\cL(X^*,Y^*)),$$ for $X$ belonging to a small region near $X^*$ and $\mu>0$ is a constant, or Error Bound conditions \citep{luo1993error}.  With this property, we cannot decrease $\cL(X,Y)$ along any direction with respect to $X$. Definition~\ref{def:crit-point} excludes the high order unstable equilibrium, which may exist due to the degeneracy of $\nabla_X^2 \cL(X,Y) $. Specifically, such a high order unstable equilibrium cannot be identified by the second order information, e.g., $$\cL(x_1,x_2,y)=x_1^3+x_2^2+y\cdot(x_1-x_2).$$ 
\begin{figure}
\centering
\includegraphics[width=1\textwidth]{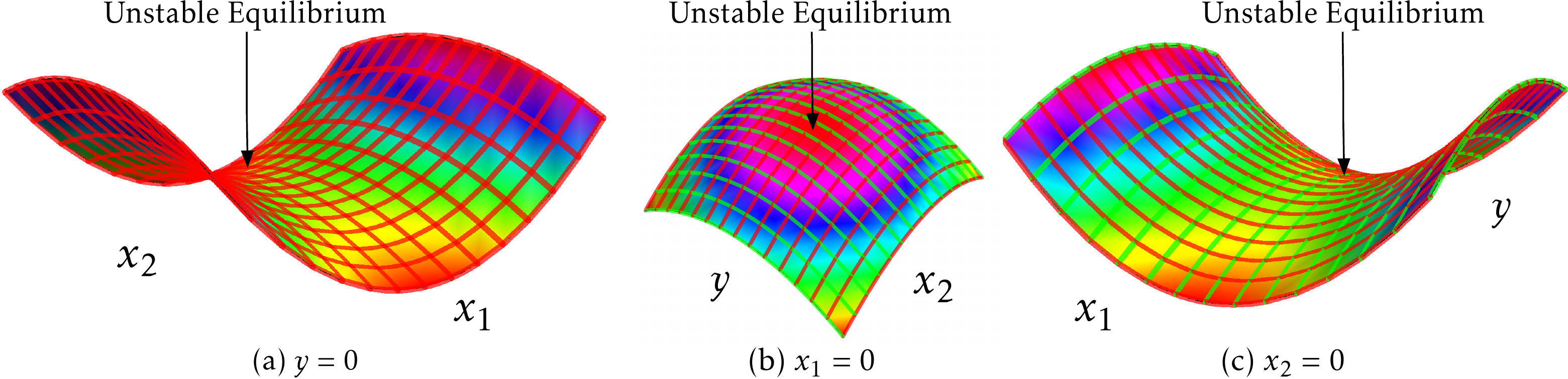}
\caption{An illustration of an unstable equilibrium: $\min_{x_1,x_2}\max_{y}\cL(x_1,x_2,y)=x_1^2-x_2^2-y^2$. Notice that $(0,0,0)$ is an equilibrium but unstable. For visualization, we show three views: (a) $\cL(x_1,x_2,0)$; (b) $\cL(0,x_2,y)$; (c) $\cL(x_1,0,y)$. The red lines correspond to $x_1$ and $x_2$, and the green one corresponds to the $y$.}
\label{illustrate}
\end{figure}
$(0,0,0)$ is an equilibrium with a positive semidefinite Hessian matrix. However, it is an unstable equilibria, since a small perturbation to $x_1$ can break this equilibrium. Such an equilibrium makes the landscape highly more complicated. Overall, we consider a specific class of Lagrangian functions throughout the rest of this paper. They enjoy the following properties:
%\noindent{1.} Strong duality holds.
\begin{itemize}[leftmargin=*]
\item All equilibria are either stable or unstable (i.e., no high order unstable equilibria);

\item All stable equilibria correspond to the global optima of the primal problem.
\end{itemize}
As mentioned earlier, the first property ensures that the second order information can identify the type of equilibria. The second property guarantees that we do not get spurious optima for \eqref{eqn-obj-1} as long as an algorithm attains a stable equilibrium. Several machine learning problems belong to this class, such as the generalized eigenvalue decomposition problem.

\section{Generalized Eigenvalue Decomposition}\label{GEV}

We consider the generalized eigenvalue (GEV) problem as a motivating example, which includes CCA, FDA, SDR, etc. as special examples. Recall its min-max formulation~\eqref{eqn:minimax}:
\begin{align*}
\min_{X\in\RR^{d\times r}}\max_{Y\in\RR^{r\times r}}\cL(X,Y)=-\tr( X^\top A X) + \langle Y,X^\top B X -
I_r \rangle.
\end{align*}
%where $A,B \in \RR^{d \times d}$ are symmetric matrices, $B$ is positive semidefinite, $X\in\RR^{d\times r}$, and $Y\in \RR^{r\times r}$.

%For notational simplicity, we have the singular value decomposition of $B$ as $B=O^B \Lambda^{B}  O^{B\top},$
%where $\Lambda^B=\diag(\Lambda^B_1,...,\Lambda^B_d)$ are eigenvalues of $B$. Let its general inverse as
%\begin{align}\label{general-inverse}
%B^{-1}=O^B  (\Lambda^{B})^{-1}  O^{B\top},
%\end{align} where 
%$(\Lambda^{B})^{-1}_i=\left\{
%\begin{array}{cc}
%(\Lambda^{B}_i)^{-1}, & \Lambda^{B}_i\neq 0\\
%0, & \Lambda^{B}_i=0
%\end{array}\right..$

Before we proceed, we impose the following assumption on the problem.
\begin{assumption}\label{assum-land}
Given a symmetric matrix $A \in \RR^{d \times d}$ and a positive definite matrix $B\in \RR^{d\times d}$, %we have the eigenvalue decomposition of $B$ as $B=O^B \Lambda^B O^{B\top}.$ 
the eigenvalues of $\tilde{A}=B^{-\frac{1}{2}} A B^{-\frac{1}{2}}$, denoted by $\lambda^{\tilde{A}}_1,...,\lambda^{\tilde{A}}_d$, satisfy $$\lambda^{\tilde{A}}_1\geq\cdots \geq \lambda^{\tilde{A}}_{r}>\lambda^{\tilde{A}}_{r+1}\geq \cdots \geq \lambda^{\tilde{A}}_d.$$
\end{assumption}

Such an eigengap assumption avoids the identifiability issue. The full rank assumption on $B$ in Assumption~\ref{assum-land} ensures that the original constrained optimization problem is bounded. This assumption can be further relaxed but require more involved analysis. We will discuss this in Appendix~\ref{sec:singular}.%Due to the space limit, we put the singular case in Appendix~\ref{sec:singular}%; for a singular $A$, the results are straightforward. 

%To characterize all equilibria, we leverage an invariant group of GEV. %Similar techniques are used in \cite{li2016symmetry} for an unconstrained matrix factorization problem. However, it cannot work  due to the min-max structure of \eqref{eqn:minimax}, which motivates us to consider a more general invariant group. 
To characterize all equilibria of GEV, we leverage the idea of an invariant group. \cite{li2016symmetry} use similar techniques for an unconstrained matrix factorization problem. However, it does not work for the Lagrangian function due to the more complicate landscape. Therefore, we consider a more general invariant group. Moreover, by analyzing the Hessian matrix of $\cL(X,Y) $ at the equilibria, we demonstrate that each equilibrium is either unstable or stable and the stable equilibria correspond to the global optima of the primal problem~\eqref{eqn:original}.  Therefore, GEV belongs to the class we defined earlier.

%
%Also, for GEV, the landscape of its Lagrangian function $\cL(X,Y)$ is closely related to the topology of $f(X)$ on the generalized Stiefel manifold \citep{absil2009optimization}, which will be further discussed later. 
%

\subsection{Invariant Group and Symmetric Property}

%We first prove the following proposition.
%\begin{proposition}\label{strong-dual}
%For the Lagrangian function of GEV, strong duality holds.
%\end{proposition}
%
%The proof of Proposition~\ref{strong-dual} is provided in the Appendix~\ref{pf:prop:strong-dual}. 

We first denote the orthogonal group in dimension $r$ as $$O(r,\RR)=\left\{ \Psi \in \RR^{r \times r} : \Psi \Psi^{\top} = \Psi^{\top} \Psi = I_r \right\}.$$  Notice that for any $\Psi\in O(r,\RR)$, $\cL(X,Y)$ in \eqref{eqn:minimax} has the same landscape with $\cL(X\Psi,\Psi^\top Y \Psi)$. This further indicates that given an equilibrium $(X,Y)$, $(X \Psi, \Psi^\top Y \Psi)$ is also an equilibrium.  This symmetric property motivates us to characterize the equilibria of $\cL(X,Y)$ with an invariant group.

We introduce several important definitions in group theory \citep{dummit2004abstract}.

%\begin{definition}\label{def:subgroup}
%	A subset of a group $\cG$ is a \textbf{\textit{subgroup}} of $\cG$ if the subset is a group under the operation induced by $\cG$.
%\end{definition}

\begin{definition}\label{def:groupaction}
Given a group $\cH$ and a set $\cX$, a map $\phi(\cdot,\cdot)$ from $\cH \times \cX$ to $\cX$ is called the {\bf group action} of $\cH$ on $\cX$
if $\phi$ satisfies the following two properties:

\noindent{\bf Identity:}
$ \phi(\mathbf{\ds{1}}, x) = x~~\forall x\in \cX$, where $\mathbf{\ds{1}}$ denotes the identity element of $\cH$.

\noindent{\bf Compatibility:}
$\phi(gh, x)=\phi(g,\phi(h,x))~~\forall g,h\in\cH,~~x\in\cX$. 
\end{definition}
 \begin{definition}\label{def:inv_group}
 		Given a function $f(x,y): \cX \times \cY \rightarrow \RR$, a group $\cH$ is a \textbf{\textit{stationary invariant group}} of $f$ with respect to two group actions of $\cH$, $\phi_1$ on $\cX$ and $\phi_2$ on $\cY$, if $\cH$ satisfies 
		\begin{align*}
		 f(x,y) = f(\phi_1(g,x),\phi_2(g,y))~~\forall x \in\cX,~~ y\in\cY, ~~\textrm{and}~~ g\in \cH.
		\end{align*}
\end{definition}
%
%\begin{definition}\label{def:directsum}
%Let $\cX$, $\cU$, and $\cV$ be three matrix spaces. If $\cX = \left\{ [u,v] : u \in \cU, v \in \cV \right\}$, then $\cX$ is the \textbf{\textit{direct sum}} of $\cU$ and $\cV$, denoted as $\cX = \cU \oplus \cV$.
%\end{definition}
For notational simplicity, we denote $\cG=O(r,\RR)$. Given the group $\cG$, two sets $\RR^{d\times r}$ and $\RR^{r\times r}$, we define a group action with $\phi_1$ of $\cG$ on $\RR^{d\times r}$ and a group action $\phi_2$ of $\cG$ on $\RR^{r\times r}$ as
\begin{align*}
\phi_1(\Psi,X)=X \Psi~~\forall \Psi\in \cG,~X\in\RR^{d\times r}\quad\textrm{and}\quad
\phi_2(g,Y)=\Psi^{-1}Y \Psi~~\forall \Psi\in \cG,~Y\in\RR^{r\times r}.
\end{align*} One can check that the orthogonal group $\cG$ is a stationary invariant group of $\cL(X,Y)$ with respect to two group actions of $\cG$, $\phi_1$ on $\RR^{d\times r}$ and $\phi_2$ on $\RR^{r\times r}$. By this invariant group, we define the equivalence relation between $(X_1,Y_1)$ and $(X_2,Y_2)$, if there exists a $\Psi \in \cG$ such that 
\begin{align}\label{eqn-eqv}
(X_1,Y_1)=(X_2\Psi,\Psi^{-1} Y_2 \Psi)=(X_2\Psi,\Psi^\top Y_2 \Psi).
\end{align}

To find all equilibria of GEV, we examine the KKT conditions of $\eqref{eqn:minimax}$:
\begin{align*}
2B X Y - 2 A  X = 0~~\textrm{and}~~X^\top B  X - I_r = 0 \Longrightarrow Y=X^\top A X=:\cD(X).
\end{align*}
Given the eigenvalue decomposition $B=O^{B} \Lambda^{B} O^{B \top}$, we denote 
\begin{align*}
\tilde{A}=(\Lambda^B)^{-\frac{1}{2}} O^{B\top}A O^B (\Lambda^B)^{-\frac{1}{2}}\quad \textrm{and}\quad \tilde{X} =  (\Lambda^B)^{\frac{1}{2}} O^{B\top} X.
\end{align*}
%Problem~\eqref{eqn:original} converts to 
%\begin{align}
%\tilde{X}^*=\argmin_{\tilde{X} \in \RR^{d\times r}} ~-\tr(\tilde{X}^\top \tilde{A} \tilde{X})\quad \textrm{s.t.} \quad \tilde{X}^\top \tilde{X}=I_r, \label{eqn:transform_obj}
%\end{align}
We then consider the eigenvalue decomposition $\tilde{A}=O^{\tilde{A}} \Lambda^{\tilde{A}} O^{\tilde{A} \top}$. The following theorem shows the connection between the equilibrium of $\cL(X,Y)$ and  the column submatrix of $O^{\tilde{A}}$, denoted as $O^{\tilde{A}}_{:,\cI}$, where $$\cI\in \cX^r_d:= \Big\{\{i_1,...,i_r\}:\{i_1,...,i_r\}\subseteq [d]\Big\}$$ is the column index set to determine a column submatrix.
\begin{theorem}[Symmetric Property]\label{thm:stationary}
	Suppose Assumption~\ref{assum-land} holds. Then $(X,\cD(X))$ is an equilibrium of $\cL(X,Y)$, if and only if $X$ can be written as $$X = (O^B(\Lambda^B)^{-\frac{1}{2}} O^{\tilde{A}}_{:,\cI})\cdot\Psi,$$ where index $\cI\in\cX^r_d$ and $\Psi\in \cG$.
	\end{theorem}
The proof of Theorem~\ref{thm:stationary} is provided in Appendix~\ref{pf:thm:stationary}. Theorem~\ref{thm:stationary} implies that  there are $\binom{d}{r}$ equilibria of $\cL(X,Y)$ under the equivalence relation given in \eqref{eqn-eqv}. Each of them corresponds to an $O^{\tilde{A}}_{:,\cI}$, where $\cI\in \cX^r_d $ is  the index set. Then whole equilibria set is generated by these $O^{\tilde{A}}_{:,\cI}$ with the transformation matrix $O^B(\Lambda^B)^{-\frac{1}{2}}$ and the invariant group action induced by $\cG$. %Such a symmetric property is also discussed in \cite{li2016symmetry}, but they only consider an unconstrained problem and use a fixed point of the invariant group to recursively find all the stationary points. Here, the min-max structure of \eqref{eqn:minimax}%, the fixed point of an invariant group (with respect to the primal variable) does not work anymore, which 
%motivates us to consider a more general invariant group.
%
%given a subset of eigenvectors corresponding to the largest $r$ eigenvalues of $\tilde{A}$, denoted as $\cS \subseteq [r]$, we can find an equilibrium of $\cL(X,Y)$. Specifically, we take its direct sum with another subset of eigenvectors corresponding to the smallest $d-r$ eigenvalues, denoted as $\tilde{\cS}\subseteq [d]\backslash [r]$. By Theorem~\ref{thm:stationary}, it only requires $|\cS|+|\tilde{\cS}|=r$. This allows us to find all equilibria in a recursive fashion. 
%

\subsection{Unstable Equilibrium vs. Stable Equilibrium}

We further identify the stable and unstable equilibria. Specifically, given $(X,Y)$ as an equilibrium of $\cL(X,Y)$, we denote the Hessian matrix of $\cL(X,Y)$ with respect to the primal variable $X$ as 
\begin{align*}
H_X \triangleq \nabla_X^2 \cL(X,Y)|_{Y=\cD(X)}\in \RR^{dr\times dr}.
\end{align*} Then we calculate the eigenvalues of $H_X$. By Definition~\ref{def:crit-point}, $(X,\cD(X))$ is unstable if $H_X$ has a negative eigenvalue; Otherwise, we analyze the local landscape at $(X,\cD(X))$ to determine whether it is stable or not.
%\vspace{0.05in}
% \left\{[O^{\dagger}_{:, \cS},O^{\dagger}_{:, \tilde{\cS}}]:|\cS|+|\tilde{\cS}|=r,\cS \subseteq [r], \tilde{\cS}\subseteq [d]\backslash[r]  \right\}
%Note that when \textbf{$B$ is nonsingular}, the feasible region of \eqref{eqn:original} is bounded. Therefore, the global optima of \eqref{eqn:original} always exist; While for a singular $B$, the global optima does not necessarily exist, which is beyond the scope of this paper.
The following theorem shows that all equilibria are either stable or unstable and demonstrates how the choice of index set $\cI$ corresponds to the unstable and stable equilibria of $\cL(X,Y)$.
\begin{theorem}\label{lem:invertB_saddle}
	Suppose Assumption~\ref{assum-land} holds, and $(X,\cD(X))$ is an equilibrium in \eqref{eqn:minimax}. By Theorem~\ref{thm:stationary}, $X$ can be represented as $X = (O^B(\Lambda^B)^{-\frac{1}{2}} O^{\tilde{A}}_{:,\cI})\cdot \Psi$ for some $\Psi\in \cG $ and $\cI\in\cX^r_d$. 
	
If $\cI \neq [r]$, then $(X,\cD(X))$ is an unstable equilibrium with
	\begin{align*}
	\lambda_{\min} (H_X) \leq \frac{2(\lambda^{\tilde{A}}_{\max \cI} - \lambda^{\tilde{A}}_{\min \cI^{\perp}})}{\| X_{:,\min \cI^{\perp}}\|_2^2}<0,%\\ &\textrm{and}~~~~~\Lambda_{\max} (H_X) \geq \frac{4 \Lambda^{\dagger}_{\min \cI }}{\| X_{:,\min \cI}\|_2^2},
	\end{align*}
	where $\lambda^{\tilde{A}}_{\max \cI}=\max_{i\in\cI}\lambda^{\tilde{A}}_i$, and $\lambda^{\tilde{A}}_{\min \cI}=\min_{i\in\cI}\lambda^{\tilde{A}}_i$,  $\lambda^{\tilde{A}}_i$ is the $i$-th leading eigenvalue of $\tilde{A}.$
	
	Otherwise, we have $H_X \succeq 0~\textrm{and}~\rk(H_X)=d\times r -r(r-1)/2.$ Moreover, $(X,\cD(X))$ is a stable equilibrium of min-max problem~\eqref{eqn:minimax}.
\end{theorem}

The proof of Theorem~\ref{lem:invertB_saddle} is provided in Appendix~\ref{pf:lem:invertB_saddle}. Theorem~\ref{lem:invertB_saddle} indicates that when $\tilde{X}=O^{\tilde{A}}_{:,[r]}$, that is, the eigenvectors of $\tilde{A}$ corresponding to the $r$ largest eigenvalues, $(X,\cD(X))$ is a stable equilibrium of $\cL(X,Y)$, where $X = (O^B(\Lambda^B)^{-\frac{1}{2}} O^{\tilde{A}}_{:,\cI}))\cdot \Psi~\textrm{for~some}~\Psi\in \cG.$ Although $H_X$ is degenerate at this equilibrium, all directions in $\textrm{Null}(H_X)$ essentially point to the primal variables of other stable equilibria. Excluding these directions, the rest all have positive curvature, which implies that this equilibrium is stable. Moreover, such an $X$ corresponds to the optima of \eqref{eqn:original}. When $\cI\neq [r]$, due to the negative curvature, these equilibria are unstable. Therefore, all stable equilibria of $\cL(X,Y)$ correspond to the global optima in $\eqref{eqn:original}$ and other equilibria are unstable, which further indicates that GEV belongs to the class we defined earlier.

%% file: algorithm.tex
\section{Stochastic Search for Online GEV}

For GEV, we propose a fully stochastic primal-dual algorithm to solve \eqref{eqn:minimax}, which only requires access to the stochastic approximations of $A$ and $B$ matrices. This is very different from other existing semi-stochastic algorithms that require to access the exact $B$ matrix~\citep{ge2016efficient}. Specifically, we propose a stochastic variant of the generalized Hebbian algorithm (GHA), also referred as Sanger's rule in existing literature \citep{sanger1989optimal}, to solve \eqref{eqn:minimax}. For online setting, accessing the exact $A$ and $B$ is prohibitive and we only get $A^{(k)}\in\RR^{d \times d}$ and $B^{(k)}\in\RR^{d \times d}$ that are independently sampled from the distribution associated with $A$ and $B$ at the $k$-th iteration. Our proposed SGHA updates primal and dual variables as follows:
%Our stochastic GHA (SGHA) algorithm is a primal-dual stochastic optimization algorithm in nature. Specifically, given $A^{(k)}\in\RR^{d \times d}$ and $B^{(k)}\in\RR^{d \times d}$ that are independently sampled from the distribution associated with $A$ and $B$ at $k$-th iteration, SGHA updates primal and dual variables as
\begin{align}
 \text{Primal Update: } &X^{(k+1)} \gets X^{(k)}-\eta\cdot \hspace{-0.25in}\underbrace{\left(B^{(k)} X^{(k)} Y^{(k)} -  A^{(k)} X^{(k)} \right),}_{\textrm{Stochastic Approximation of $\nabla_{X}\cL(X^{(k)},Y^{(k)})$}} \label{alg:primal_update}\\
 \text{Dual Update: }
&Y^{(k+1)}  \gets \hspace{-0.45in}\underbrace{X^{(k)\top}  A^{(k)} X^{(k)},}_{\textrm{Stochastic Approximation of $X^{(k)\top}AX^{(k)}$}} \label{alg:dual_update}
\end{align}
where $\eta >0$ is a step size parameter.  Note that the primal update is a stochastic gradient descent step, while the dual update is motivated by the KKT conditions of \eqref{eqn:minimax}. SGHA is simple and easy to implement. The constraint is naturally handled by the dual update. Further, motivated by the the landscape of GEV, we analyze the algorithm by diffusion approximations and obtain the asymptotical sample complexity.

%The initial solution $X^{(0)} \in \RR^{d \times r}$ only needs to be chosen as a random vector with each entry independently following a mean zero and variance $\frac{1}{d}$ normal distribution. 

\subsection{Numerical Evaluations}
We first provide numerical evaluations to illustrate the effectiveness of SGHA, and then provide an asymptotic convergence analysis of SGHA. We choose $d=500$ and select three different settings: 

\begin{itemize}[leftmargin=*]
\item {$\boldsymbol{\textrm{\bf Setting} (1): }$}  $\eta=10^{-4}$, $r=1$, $A_{ii} = 1/100$~~$\forall i \in [d]$, $A_{ij} = 0.5/10$ and $B_{ij}=0.5^{|i-j|}/3$~~$\forall i\neq j$;
\item {$\boldsymbol{\textrm{\bf Setting} (2): }$} $\eta=5\times 10^{-5}$, $r=3$, and randomly generate an orthogonal matrix $U\in\RR^{d\times d}$ such that $A= U\cdot\diag(1,1,1,0.1,...,0.1)\cdot U^\top$ and $B=U\cdot\diag(2,2,2,1,...,1)\cdot U^\top$;
\item {$\boldsymbol{\textrm{\bf Setting} (3): }$} $\eta=2.5\times 10^{-5}$, $r=3$, and randomly generate two orthogonal matrices $U,V\in\RR^{d\times d}$ such that $A= U\cdot\diag(1,1,1,0.1,...,0.1)\cdot U^\top$ and $B=V\cdot \diag(2,2,2,1,...,1)\cdot V^\top$.
\end{itemize}
%
%\noindent Setting (1): $\eta=10^{-4}$, $r=1$, $A_{ii} = 1/100$~~$\forall i \in [d]$, $A_{ij} = 0.5/10$ and $B_{ij}=0.5^{|i-j|}/3$~~$\forall i\neq j$; \\
%\noindent Setting (2): $\eta=5\times 10^{-5}$, $r=3$, and randomly generate an orthogonal matrix $U\in\RR^{d\times d}$ such that $A= U\cdot\diag(1,1,1,0.1,...,0.1)\cdot U^\top$ and $B=U\cdot\diag(2,2,2,1,...,1)\cdot U^\top$;\\
%\noindent Setting (3): $\eta=2.5\times 10^{-5}$, $r=3$, and randomly generate two orthogonal matrices $U,V\in\RR^{d\times d}$ such that $A= U\cdot\diag(1,1,1,0.1,...,0.1)\cdot U^\top$ and $B=V\cdot \diag(2,2,2,1,...,1)\cdot V^\top$.

\begin{figure}
	\centering
	\subfigure[Setting (1)]{
		\includegraphics[width=0.3\textwidth]{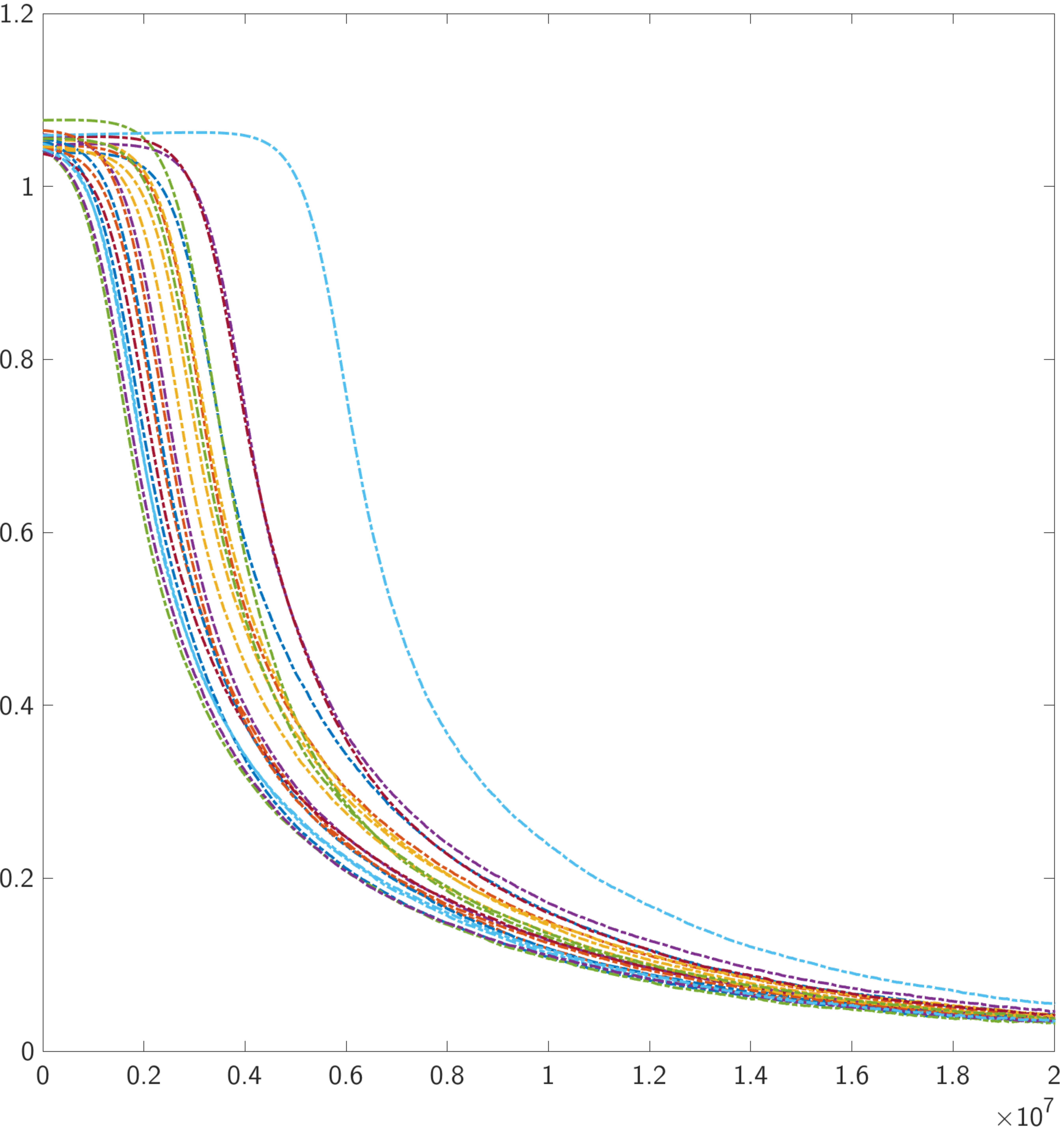}\label{setting-1}
	}
	\subfigure[Setting (2)]{
		\includegraphics[width=0.3\textwidth]{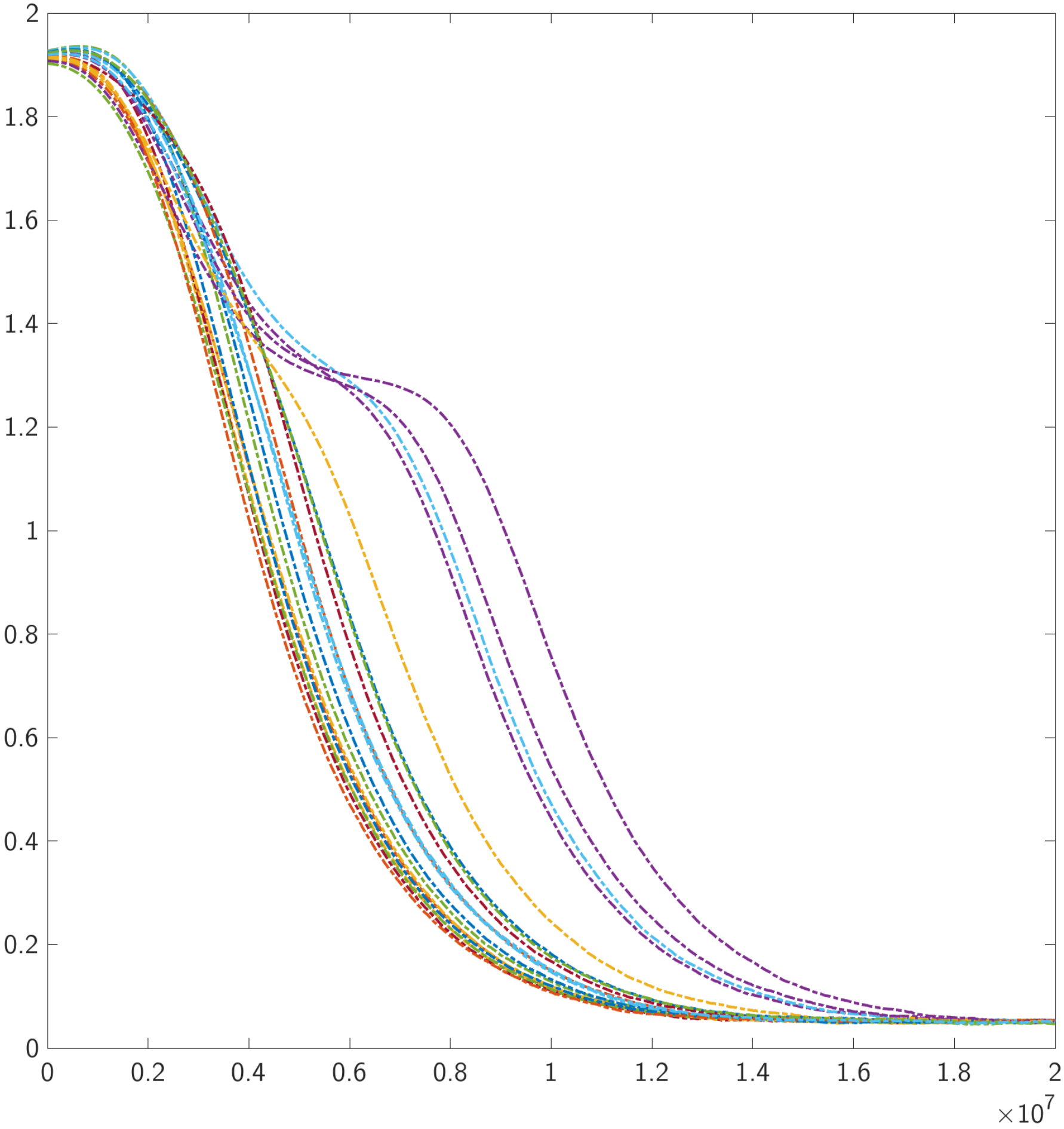}\label{setting-2}
	}
	\subfigure[Setting (3)]{
		\includegraphics[width=0.3\textwidth]{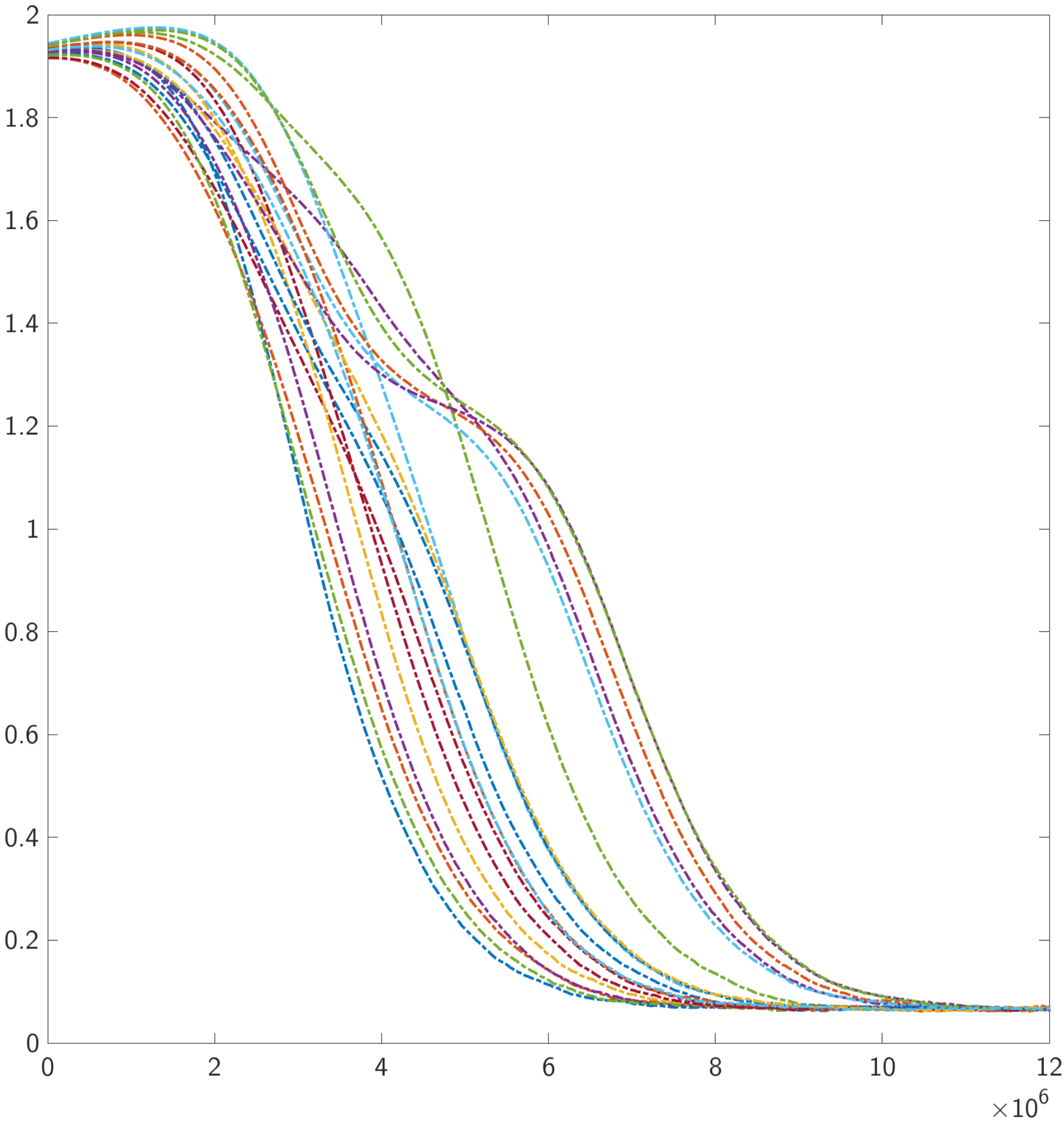}\label{setting-3}
	}
	\caption{Plots of the optimization error $\norm{B^{1/2}X^{(t)}X^{(t)\top}B^{1/2}-B^{1/2}X^{*}X^{*\top}B^{1/2}}_{\rm F}$ over SGHA iterations on synthetic data of 20 random data generations under different settings of parameters.}
\label{numerical-results}
\end{figure}

At the $k$-th iteration of SGHA, we independently sample $40$ random vectors from $N(0,A)$ and $N(0,B)$ respectively. Accordingly, we compute the sample covariance matrices $A^{(k)}$ and $B^{(k)}$ as the approximations of $A$ and $B$. We repeat numerical simulations under each setting for $20$ times using random data generations, and present all results in Figure \ref{numerical-results}. The horizontal axis corresponds to the number of iterations, and the vertical axis corresponds to the optimization error $$\norm{B^{1/2}X^{(t)}X^{(t)\top}B^{1/2}-B^{1/2}X^{*}X^{*\top}B^{1/2}}_{\rm F}.$$ Our experiments indicate that SGHA converges to a global optimum in all settings.

%\subsection{Hardness of Theoretical Analysis}

\subsection{Convergence Analysis for Commutative $A$ and $B$}

As a special case, we first prove the convergence of SGHA for GEV with $r=1$, and $A$ and $B$ are commutative. We will discuss more on noncommutative cases and $r>1$ in the next section. Before we proceed, we introduce our assumptions on the problem.
\begin{assumption}\label{assmp}
We assume that the following conditions hold:
\begin{itemize}[leftmargin=*]
\item {\bf (a):} $A^{(k)}$'s and $B^{(k)}$'s are independently sampled from two different distributions $\cD_A$ and $\cD_B$ respectively, where $\EE A^{(k)} = A$ and $\EE B^{(k)}=B \succ 0$; 
\item{\bf(b):} $A$ and $B$ are commutative, i.e., there exists an orthogonal matrix $O$ such that $A = O\Lambda^{A}O^\top$ and $B = O\Lambda^{B}O^\top$, where $\Lambda^A=\diag(\lambda_1,...,\lambda_d)$ and $\Lambda^B=\diag(\mu_1,...,\mu_d)$ are diagonal matrices with $\lambda_j\neq 0$; 
\item {\bf(c):} $A^{(k)}$ and $B^{(k)}$ satisfy the moment conditions, that is, for some generic constants $C_0$ and $C_1$, $\EE {\norm{A^{(k)}}_2^2}\leq C_0$ and $\EE {\norm{B^{(k)}}_2^2}\leq C_1$.
\end{itemize}
\end{assumption}
Note that (a) and (c) in \eqref{assmp} are mild, but (b) is stringent. For convenience of analysis, we combine \eqref{alg:primal_update} and \eqref{alg:dual_update} as
\begin{align}\label{update}
 X^{(k+1)} \gets  X^{(k)} - \eta \big(B^{(k)} X^{(k)} X^{(k)\top} \hspace{-0.035in}- I_d \big) A^{(k)} X^{(k)}.
\end{align}
We remark that \eqref{update} is very different from existing optimization algorithms over the generalized Stiefel manifold. Specifically, computing the gradient over the generalized Stiefel manifold requires $B^{-1}$, which is not allowed in our setting. For notational convenience, we further denote
\begin{align*}
\Lambda=(\Lambda^B)^{-\frac{1}{2}} \Lambda^A (\Lambda^B)^{-\frac{1}{2}}=\diag\left(\frac{\lambda_1}{\mu_1},...,\frac{\lambda_d}{\mu_d}\right)=:\diag(\beta_1,\cdots,\beta_d).
\end{align*} 
Without loss of generality, we assume $ \beta_1>\beta_2 \geq \beta_3 \geq\cdots\geq\beta_d,$ and $\beta_i\neq 0~~\forall i\in[d]$. Note that $\mu_i$ and $\lambda_i$, however, are not necessarily to be monotonic. We denote 
\begin{align*}
\mu_{\min}=\min_{i \neq 1} \mu_i,\quad \mu_{\max}=\max_{i\neq 1} \mu_i,\quad \textrm{and}\quad {\sf gap}=\beta_1-\beta_2.
\end{align*}
 Denote $W^{(k)}=(\Lambda^{B})^\frac{1}{2}O X^{(k)}$.  One can verify that \eqref{update} can be rewritten as follows:
\begin{align}\label{equivalent}
 W^{(k+1)}&\gets W^{(k)}-\eta\Big((\Lambda^B)^{\frac{1}{2}}\hat{\Lambda}_B^{(k)}(\Lambda^B)^{-\frac{1}{2}}\cdot W^{(k)}W^{(k)\top} - \Lambda^B \Big) \cdot\tilde{\Lambda}^{(k)}W^{(k)},
\end{align}
where $\hat{\Lambda}_B^{(k)}= O^{\top} B^{(k)}O$ and $\tilde{\Lambda}^{(k)}=O^{\top} B^{-\frac{1}{2}} A^{(k)}B^{-\frac{1}{2}}O.$ Note that $W^*=(1,\underbrace{0,0,...,0}_{(d-1)})^\top$ corresponds to the optimal solution of \eqref{eqn:original}.

By diffusion approximation, we show that our algorithm converges through three Phases: 
\begin{itemize}[leftmargin=*]
\item{\bf Phase I:} Given an initial near a saddle point, we show that after rescaling of time properly, the algorithm can be characterized by a stochastic differential equation (SDE). Such an SDE further implies our algorithm can escape from the saddle fast; 
\item{\bf Phase II:}  We show that away from the saddle, the trajectory of our algorithm can be approximated by an ordinary differential equation (ODE); 
\item{\bf Phase III:} We first show that after Phase II, the norm of solution converges to a constant. Then, the algorithm can be characterized by an SDE, like Phase I. By the SDE, we analyze the error fluctuation when the solution is within a small neighborhood of the global optimum. 
\end{itemize}
Overall, we obtain an asymptotic sample complexity.

\noindent{\bf ODE Characterization:} To demonstrate an ODE characterization for the trajectory of our algorithm, we introduce a continuous time random process 
\begin{align*}
w^{(\eta)}(t):=W^{(k)},
\end{align*}
where $k=\lfloor \frac{t}{\eta}\rfloor$ and $\eta$ is the step size in \eqref{update}. For notational simplicity, we drop $(t)$ when it is clear from the context. Instead of showing a global convergence of $w^{(\eta)}$, we show that the quantity 
\begin{align*}
v^{(\eta)}_{i,j}=\frac{(w^{(\eta)}_i)^{\mu_j}}{(w^{(\eta)}_j)^{\mu_i}}
\end{align*} converges to an exponential decay function, where $v^{(\eta)}_i$ is the $i$-th component (coordinate) of $w^{(\eta)}$.

\begin{lemma}\label{lem:ode}
	Suppose that Assumption~\ref{assmp} holds and the initial solution is away from any saddle point, i.e., given pre-specified constants, $\tau>0$ and $\delta<\frac{1}{2}$, there exist $i,j$ such that 
\begin{align*}
i\neq j, \quad |w^{(\eta)}_j|> \tau, \quad \textrm{and}\quad |w^{(\eta)}_i|>\eta^{\frac{1}{2}+\delta}.
\end{align*}
 As $\eta \rightarrow 0$, $v^{(\eta)}_{k,j}$ weakly converges to the solution of the following ODE:
\begin{align}\label{equ:ode}
	d x_{k,j} = x_{k,j} \cdot \left( \mu_j\mu_k ( \beta_k-\beta_j)\right)d t~~\forall k\neq j.
\end{align}
\end{lemma}
The proof of Lemma~\ref{lem:ode} is provided in Appendix~\ref{pf:lem:ode}. Lemma~\ref{lem:ode} essentially implies the global convergence of SGHA. Specifically, the solution of \eqref{equ:ode} is 
\begin{align*}
x_{k,j}(t)=x_{k,j}(0)\cdot\exp\left( \mu_j\mu_k \left(\beta_k-\beta_j\right)t \right)~~\forall k\neq j,
\end{align*}
where $x_{k,j}(0)$ is the initial value of $v^{(\eta)}_{k,j}$. In particular, we consider $j=1$. Then, as $t\rightarrow \infty$, the dominating component of $w$ will be $w_1$. 

The ODE approximation of the algorithm implies that after long enough time, i.e., $t$ is large enough, the solution of the algorithm can be arbitrarily close to a global optimum. Nevertheless, to obtain the asymptotic ``convergence rate'', we need to study the variance of the trajectory at time $t$. Thus, we resort to the following SDE-based approach for a more precise characterization.

\noindent{\bf SDE Characterization:} % The ODE approximation of the algorithm implies that after large enough number of iterations, i.e., $t$ is large enough, the solution of the algorithm can be arbitrarily close to a global optimum. Nevertheless, to obtain the ``convergence rate'', we need to study the variance of the trajectory at time $t$. 
We notice that such a variance with order $\cO(\eta)$ vanishes as $\eta \rightarrow 0$. To characterize this variance, we rescale the updates by a factor of $\eta^{-\frac{1}{2}}$, i.e., by defining a new process as $z^{(\eta)}=\eta^{-\frac{1}{2}}w^{(\eta)}$. After rescaling, the variance of $z^{(\eta)}$ is of order $\cO(1)$.
The following lemma characterizes how the algorithm escapes from the saddle, i.e., $w^{(\eta)}(0)\approx e_i$, where $i\neq 1$, in Phase I.
\begin{lemma}\label{lem:sadd}
Suppose Assumption~\ref{assmp} holds and the initial is close to a saddle point , i.e., $z^{(\eta)}_j(0) \approx \eta^{-\frac{1}{2}}$ and $z^{(\eta)}_i(0) \approx 0$ for $i\neq j$. Then for any $C>0,$ there exist $\tau>0$ and $\eta'>0$ such that
\begin{align}\label{sde:con}
\sup_{ \eta<\eta'}\PP(\sup_t |z^{(\eta)}_i(t)|\leq C)\leq 1-\tau.
\end{align}

%that is, given pre-specified constants $\delta<\frac{1}{2}$ and $D$, there exists an $i\in[d]\backslash\{1\}$ such that 
%\begin{align*}
%|w_i^{(\eta)}-1|\leq D \eta^{\frac{1}{2}+\delta} \quad  \textrm{and} \quad |w_j^{(\eta)}|\leq D \eta^{\frac{1}{2}+\delta}\quad\forall j\neq i.
%\end{align*}
% As  $\eta \rightarrow 0$, then  $z^{(\eta)}_i$ weakly converges to the solution of the following SDE:
% and $\left(\hat{\Lambda}_B \right)_{i,1}$ is the entry of $\hat{\Lambda}_B$ corresponding to the $i$-th row and $1$-st column, similar definition for $\tilde{\Lambda}_{1,1}$ and $\Lambda_{i,1}.
\end{lemma}
Here we provide the proof sketch and leave the whole proof of Lemma~\ref{lem:sadd} in Appendix~\ref{pf:lem:sadd}.
\begin{proof}[Proof Sketch]
 We prove this argument by contradiction. Assume the conclusion does not hold, that is there exists a constant $C>0,$ such that for any  $\eta'>0$ we have $$\sup_{\eta\leq \eta'}\PP(\sup_t |z^{(\eta)}_i(t)|\leq C)=1.$$ That implies there exists a sequence $\{\eta_n\}_{n=1}^\infty$ converging to $0$ such that 
\begin{align}\label{eq_contra}
\lim_{n\rightarrow\infty}\PP(\sup_t |z^{(\eta_n)}_i(t)|\leq C)= 1.
\end{align}
Then we show  $\{z^{(\eta_n)}_i(\cdot)\}_n$ is tight and thus converges weakly. Furthermore, $\{z^{(\eta_n)}_i(\cdot)\}_n$ weakly converges to a stochastic differential equation,
	\begin{align}\label{equ:sadd}
	d z_j(t) =\left( -\beta_j\mu_i \cdot z_i +\lambda_i z_i \right)dt+\sqrt{G_{j,i}}d B(t)~~\textrm{for $j\in[d]\backslash\{i\}$,}
	\end{align}
	where ${G_{j,i}}=\EE\Big(\big(\hat{\Lambda}_B^{(k)}\big)_{j,i}\cdot \sqrt{\mu_j/ \mu_i}\cdot \tilde{\Lambda}_{i,i}-\mu_j\tilde{\Lambda}_{j,i}\Big)^2$ and $B(t)$ is a standard Brownian motion.
We compute the solution of this stochastic differential equation and then show \eqref{sde:con} holds. 
\end{proof}

 Note that \eqref{equ:sadd} is a Fokker-Plank equation, whose solution is an Ornstein-Uhlenbeck (O-U) process \citep{doob1942brownian} as follows:
\begin{align}\label{O-U process-1}
&z_j(t)=\Big[\underbrace{z_j(0)+\sqrt{G_{j,i}}\int_0^t \exp\left[\mu_j \left(\beta_i -\beta_j \right)s\right] dB(s)}_{Q_1}\Big]\cdot \exp\left[ -\mu_j \left(\beta_i- \beta_j \right)t \right].
\end{align}
We consider $j=1$. Note that $Q_1$  is essentially a random variable with mean $z_{j}(0)$ and variance smaller than $\frac{G_{1,i}\mu_1}{2(\beta_1-\beta_i)}$. However, the larger $t$ is, the closer its variance gets to this upper bound. Moreover, the term $\exp\big[\mu_1(\beta_1-\beta_i)t\big]$ essentially amplifies $Q_1$ by a factor exponentially increasing in $t$. This tremendous amplification forces $z_{1}(t)$ to quickly get away from $0$, as $t$ increases, which indicates that the algorithm will escape from the saddle. Further, the following lemma characterizes the local behavior of the algorithm near the optimal.

\begin{lemma}\label{lem:sde}
	Suppose that Assumption~\ref{assmp} holds and the initial solution is close to an optimal solution, that is, given pre-specified constants $\kappa$ and $\delta<\frac{1}{2}$, we have $\frac{|w^{(\eta)}_1|^2}{\norm{w^{(\eta)}}_2^2}>1-\kappa \eta^{1+2\delta}$. As  $\eta \rightarrow 0$, then we have $\norm{w^{(\eta)}(t)}_2\xrightarrow{t\rightarrow \infty} 1$ and $z^{(\eta)}_i$ weakly converges to the solution of the following SDE:
	\begin{align}\label{equ:sde}
	d z_i(t) =\left( -\beta_1\cdot \mu_i z_i +\lambda_i z_i \right)dt+\sqrt{G_{i,1}}d B(t)~~\textrm{for $i\neq 1$,}
	\end{align}
	where $G_{i,1}=\EE\big((\hat{\Lambda}_B )_{i,1}\cdot \sqrt{\mu_i/\mu_1}\cdot \tilde{\Lambda}_{1,1}-\mu_i\Lambda_{i,1}\big)^2$, and $B(t)$ is a standard Brownian motion.% and $\left(\hat{\Lambda}_B \right)_{i,1}$ is the entry of $\hat{\Lambda}_B$ corresponding to the $i$-th row and $1$-st column, similar definition for $\tilde{\Lambda}_{1,1}$ and $\Lambda_{i,1}.$
\end{lemma}
The proof of Lemma~\ref{lem:sde} is provided in Appendix~\ref{pf:lem:sde}. The solution of \eqref{equ:sde} is as follows:
\begin{align}\label{O-U process-2}
&z_i(t)=\sqrt{G_{i,1}}\int_0^t \exp\left[\mu_i\left(\beta_1  -\beta_i \right)\left(s-t\right)\right] dB(s)+z_i(0) \cdot \exp\left[  -\mu_i\left(\beta_1  -\beta_i \right)t \right].
\end{align}
Note the second term of the right hand side in \eqref{O-U process-2} decays to 0, as time $t\rightarrow \infty$. The rest is a pure random walk. Thus, the fluctuation of $z_i(t)$ is essentially the error fluctuation of the algorithm after sufficiently long time.%, which further implies $w^{(\eta)}(t)$ converges to a random process, denoted by $w(t)$.

Combining Lemma~\ref{lem:ode},~\ref{lem:sadd}, and~\ref{lem:sde}, %and the choice of $\eta$ above, 
we obtain the following theorem.

\begin{theorem}\label{thm:num-iter-step-size}
Suppose Assumption~\ref{assmp} holds. Given a sufficiently small error $\epsilon>0$,~$\phi=\sum_{i=1}^dG_{i,1}$, and%learning rate as
\begin{align*}
\eta \asymp \frac{\epsilon \cdot \mu_{\min}\cdot{\sf gap}}{\phi},
\end{align*}
we need
\begin{align}\label{num}
T\asymp\frac{\mu_{\max}/\mu_{\min}}{\mu_1\cdot\sf{gap}}\log\left(\eta^{-1}\right)
\end{align}
such that with probability at least $\frac{5}{8}$, $\norm{w(T)-W^*}_2^2\leq \epsilon$, where $W^*$ is the optima of \eqref{eqn:original}.
\end{theorem}
The proof of Theorem~\ref{thm:num-iter-step-size} is provided in Appendix~\ref{pf:thm:num-iter-step-size}. Theorem~\ref{thm:num-iter-step-size} implies that asymptotically, our algorithm yields an iterations of complexity:
\begin{align*}
N\asymp \frac{T}{\eta}\asymp\frac{\phi\cdot \mu_{\max}/\mu_{\min}}{\epsilon\cdot\mu_1\cdot\mu_{\min}\cdot\sf{gap}^2}\log\left(\frac{\phi}{\epsilon\cdot\mu_{\min}\cdot \sf{gap}}\right),
\end{align*}
which not only depends on the gap, i.e., $\beta_1-\beta_2$, but also depends on $\frac{\mu_{\max}}{\mu_{\min}}$, which is the condition number of $B$ in the worst case. As can be seen, for an ill-conditioned $B$, the problem~\eqref{eqn:original} is more difficult to solve.%We then present a proof sketch of Theorem \ref{thm:num-iter-step-size}, and further details are deferred to Appendix~\ref{pf:thm:num-iter-step-size}.

%% file: Discussion.tex
\subsection{When $A$ and $B$ are Noncommutative?}

Unfortunately, when $A$ and $B$ are noncommutative, the analysis is more difficult, even for $r=1$. Recall that the optimization landscape of the Lagrangian function in \eqref{eqn:minimax} enjoys a nice geometric property: At an unstable equilibrium, the negative curvature with respect to the primal variable encourages the algorithm to escape. Specifically, suppose the algorithm is initialized at an unstable equilibrium $(X^{(0)},Y^{(0)})$, the descent direction for $X^{(0)}$ is determined by the eigenvectors of 
\begin{align*}
H_{X^{(0)}} = A + Y^{(0)}B
\end{align*} associated with the negative eigenvalues. After one iteration, we obtain $(X^{(1)},Y^{(1)})$. The Hessian matrix becomes 
\begin{align*}
H_{X^{(1)}} = A + Y^{(1)}B.
\end{align*}
Since $Y^{(1)}=X^{(0)\top} A^{(0)} X^{(0)}$ is a stochastic approximation, the random noise can make $Y^{(1)}$ significantly different from $Y^{(0)}$. Thus, the eigenvectors of  $H_{X^{(1)}}$ associated with the negative eigenvalues can be also very different from those of $H_{X^{(0)}}$. This phenomenon can seriously confuse the algorithm about the descent direction of the primal variable. We remark that such an issue does not appear if we assume $A$ and $B$ are commutative. We suspect that this is very likely an artifact of our proof technique, since our numerical experiments have provided some empirical evidences of the convergence of SGHA.

\section{Discussion}

Here we briefly discuss a few related works:
\begin{itemize}[leftmargin=*]
\item {\cite{li2016symmetry}} propose a framework for characterizing the stationary points in the unconstrained nonconvex matrix factorization problem, while our studied generalized eigenvalue problem is constrained. Different from their analysis, we analyze the optimization landscape of the corresponding Lagrangian function. When characterize the stationary points, we need to take both primal and dual variables into consideration, which is technically more challenging. 

%Our analysis is also different from existing literature on manifold optimization. For example, \cite{absil2009optimization} only characterize the quotient space regarding the optimum, while we characterize not only optima but also unstable equilibria for GEV by group theory. %even with $B$ being potentially singular. 
\item {\cite{ge2016efficient}} also consider the (off-line) generalized eigenvalue problem but in a finite sum form. Unlike our studied online setting, they access exact $A$ and $B$ in each iteration. Specifically, they need to access exact $A$ and $B$ to compute an approximate inverse of $B$ to find the descent direction. Meanwhile, they also need a modified Gram Schmidt process, which also requires accessing exact $B$, to maintain the solution on the generalized Stiefel manifold (defined by $X^\top B X=I_r$ via exact $B$, \cite{mishra2016riemannian}). Our proposed stochastic search, however, is a full stochastic primal-dual algorithm, which neither require accessing exact $A$ and $B$, nor enforcing the the primal variables to stay on the manifold.
\end{itemize}

%% file: Append-1.tex
%!TEX root = ./primal_dual.tex

\newpage
\onecolumn
\appendix
\section{Proofs for Determining Stationary Points}

%
%\subsection{Proof of Proposition~\ref{strong-dual}}\label{pf:prop:strong-dual}
%\begin{proof}
%Suppose $\rk(B)=m$. The Lagrangian formula of GEV~\eqref{eqn:original} is
%\begin{align}\label{GEV-lag}
%\cL(X,Y)=-\tr(X^\top A X)+\langle Y,X^\top BX-I_r\rangle.
%\end{align}
%On the one hand, if $X^\top B X -I_r \neq 0$,  we can choose a sequence of $\{Y_i\}$ to make $\langle Y_i,X^\top BX-I_r\rangle\rightarrow \infty$.
%If $X^\top B X-I_r=0$, then the problem becomes the primal problem~\eqref{eqn:original}. 
%
%On the other hand, let $X=B^{-\frac{1}{2}}W$, where $B^{-\frac{1}{2}}$ is the square root of the general inverse of the matrix $B$. Then \eqref{GEV-lag} becomes
%\begin{align*}
%\tilde{\cL}(W,Y)=\cL(X,Y)=-\tr(W^\top B^{-\frac{1}{2}}A B^{-\frac{1}{2}}W)+\langle Y, W^\top W-I_r\rangle.
%\end{align*}
%Then we calculate the Hessian matrix of $\tilde{\cL}(W,Y)$ with respect to $W$ as follow:
%\begin{align*}
%\nabla^2_{W}\tilde{\cL}(W,Y)= -2 I_r \otimes (B^{-\frac{1}{2}}A B^{-\frac{1}{2}} )+ 2Y\otimes \diag(I_m,0_{d-m}).
%\end{align*}
%So, if $\nabla^2_W\tilde{\cL}(W,Y)\succeq 0$, then the $\min_X \cL(X,Y)=\min_W \tilde{\cL}(W,Y)=\langle Y, -I_r\rangle=-\tr(Y)$; otherwise, it equals $-\infty$. Further, we can obtain that $\max_Y -\tr(Y)=\sum_{i=1}^r -\sigma_i(B^{-\frac{1}{2}}AB^{-\frac{1}{2}})$, where $\sigma_i(\cdot)$ takes the $i$-th leading eigenvalue. This is also the optimal value of the primal problem.
%Therefore, strong duality holds.
%\end{proof}
%

\subsection{Proof of Theorem~\ref{thm:stationary}}\label{pf:thm:stationary}
\begin{proof}
	 Remind that the eigendecomposition of $\tilde{A}$ is $(\Lambda^B)^{-\frac{1}{2}} O^{B\top}A O^B (\Lambda^B)^{-\frac{1}{2}}= O^{\tilde{A}} \Lambda^{\tilde{A}} (O^{\tilde{A}})^{\top}$. Given the eigendecomposition of $B$ is $B = O^{B} \Lambda^{B} (O^{B})^{\top}$, we can write $B^{-1}$ as
	\begin{align*}
		B^{-1} = O^B (\Lambda^B)^{-1} (O^B)^{\top}.
	\end{align*}
	
%	From the construction of $\tilde{X} \in \{O^{\tilde{A}}_{:,\cI}: |\cI|=r\}$, we have $\tilde{X} = [O^{\tilde{A}}_{:,\cS}~ O^{\tilde{A}}_{:,\tilde{\cS}}]$. Without loss of generality, we assume that $\cS=\cI$. 

We denote $\tilde{X}$ as $\tilde{X} = O^{\tilde{A}}_{:,\cI}$ for some $\cI \subseteq [d]$ with $|\cI| = r$. For $X = (B^{-1/2} O^{\tilde{A}}_{:,\cI})\cdot \Psi ,$ where $\Psi\in \cG$. It is easy to see that $\nabla_Y \cL(X,Y)=0$. Ignore the constant 2 in the gradient $\nabla_X \cL(X,Y)$ for convenience, we have,
	\begin{align*}
		\nabla_X \cL(X,Y) &= -(I_d - BXX^\top) A X = -(I_d - B B^{-1/2} O^{\tilde{A}}_{:,\cI} (O^{\tilde{A}}_{:,\cI})^{\top} B^{-1/2} ) A B^{-1/2} O^{\tilde{A}}_{:,\cI} \\
		& = -A B^{-1/2} O^{\tilde{A}}_{:,\cI} + B^{1/2} O^{\tilde{A}}_{:,\cI} (O^{\tilde{A}}_{:,\cI})^{\top} O^{\tilde{A}} \Lambda^{\tilde{A}} (O^{\tilde{A}})^{\top} O^{\tilde{A}}_{:,\cI} \\
		& = -B^{1/2} O^{\tilde{A}} \Lambda^{\tilde{A}} (O^{\tilde{A}})^{\top} O^{\tilde{A}}_{:,\cI} + B^{1/2} O^{\tilde{A}}_{:,\cI} \Lambda^{\tilde{A}}_{\cI,\cI} \\
		&= -B^{1/2} O^{\tilde{A}} \Lambda^{\tilde{A}}_{:\cI} + B^{1/2}O^{\tilde{A}}_{:,\cI} \Lambda^{\tilde{A}}_{\cI,\cI} = 0.
	\end{align*}
	
	Next we show that if $X$ is not as specified, then $\nabla_X \cL(X,Y) \neq 0$. We only need to show that if $\tilde{X} = [O^{\tilde{A}}_{:,\cS}, ~\phi] \Psi$, where $\cS \subseteq [d]$ with $|\cS| = r-1$ and $\phi = c_1 O^{\tilde{A}}_{:,i} + c_2 O^{\tilde{A}}_{:,j}$ with $i,j \not\in \cS$, $i \neq j$, $c_1^2 + c_2^2 = 1$, and $c_1,c_2 \neq 0$, then we have $\nabla_X \cL(X,Y) \neq 0$. The general scenario can be induced from this basic setting. It is easy to see that such an $X = B^{-1/2} \tilde{X}$ satisfies the constraint,
	\begin{align*}
		X^{\top} B X = \Psi^{\top} [O^{\tilde{A}}_{:,\cS}, ~\phi]^{\top} B^{-1/2} B B^{-1/2}[O^{\tilde{A}}_{:,\cS}, ~\phi] \Psi = \Psi^{\top} \left[ \begin{array}{cc}
			I_{r-1} & 0_{(r-1) \times 1} \\
			0_{1 \times (r-1)} & \phi^{\top} \phi
		\end{array} \right] \Psi = I_r,
	\end{align*}
	where the last equality follow from $\phi^{\top} \phi = c_1^2 + c_2^2 = 1$. 
	
	Plugging such an $X$ into the gradient, we have
	\begin{align*}
		\nabla_X \cL(X,Y) &= -(I_d - B X X^\top) AX = -(I_d - B B^{-1/2} [O^{\tilde{A}}_{:,\cS}, ~\phi] [O^{\tilde{A}}_{:,\cS}, ~\phi]^{\top} B^{-1/2}) A B^{-1/2} [O^{\tilde{A}}_{:,\cS}, ~\phi] \Psi \\
		&= -B^{1/2} (O^{\tilde{A}}_{:,\cS^{\perp}} (O^{\tilde{A}}_{:,\cS^{\perp}})^{\top} - \phi \phi^{\top}) O^{\tilde{A}} \Lambda^{\tilde{A}} [(I_d)_{\cS}, ~c_1 e_i + c_2 e_j] \Psi \\
		&= -B^{1/2} [0_{d \times (r-1)}, ~O^{\tilde{A}}_{:,\cS^{\perp}} \Lambda^{\tilde{A}}_{\cS^{\perp},:}(c_1 e_i + c_2 e_j)]\Psi + [0_{d \times (r-1)},~\phi (c_1^2 \lambda^{\tilde{A}}_{i}+c_2^2 \lambda^{\tilde{A}}_{j})]\Psi \\
		&= -B^{1/2} [0_{d \times (r-1)}, ~c_1 c_2^2(\lambda^{\tilde{A}}_{i} + \lambda^{\tilde{A}}_{j}) O^{\tilde{A}}_{:,i} + c_2 c_1^2(\lambda^{\tilde{A}}_{j} - \lambda^{\tilde{A}}_{i}) O^{\tilde{A}}_{j,j}] \Psi \neq 0,
	\end{align*}
	where the last $\neq$ is from $c_1,c_2 \neq 0$, $c_1^2+ c_2^2 = 1$, $\lambda^{\tilde{A}}_{j} \neq \lambda^{\tilde{A}}_{j}$ for $i \neq j$.
\end{proof}

\subsection{Proof of Theorem~\ref{lem:invertB_saddle}} \label{pf:lem:invertB_saddle}

%Without loss of generality, we set $B^{1/2} = O^{B} (\Lambda^{B})^{1/2} O^{B \top}$, thus both $B^{1/2}$ and $B^{-1/2}$ are symmetric. Plugging in $X = B^{-1/2} O^{\tilde{A}}_{\cS} \Psi$, we have
%\begin{align*}
%\nabla_X \cL(X,Y) &= (I_d - B X X^\top) AX = \left( I_d -  B^{1/2} O^{\tilde{A}}_{:,\cS} \Psi \Psi^{\top} (O^{\tilde{A}}_{:,\cS})^{\top} (B^{-1/2})^{\top} \right) A B^{-1/2} O^{\tilde{A}}_{:,\cS} \Psi \\
%&= A B^{-1/2} O^{\tilde{A}}_{:,\cS} \Psi - B^{1/2} O^{\tilde{A}}_{:,\cS} (O^{\tilde{A}}_{:,\cS})^{\top} O^{\tilde{A}} \Lambda^{\tilde{A}} (O^{\tilde{A}})^{\top} O^{\tilde{A}}_{:,\cS} \Psi  \\
%&= B^{1/2} O^{\tilde{A}} \Lambda^{\tilde{A}} (O^{\tilde{A}})^{\top} O^{\tilde{A}}_{:,\cS} \Psi - B^{1/2} O^{\tilde{A}}_{:,\cS} \Lambda^{\tilde{A}}_{\cS,\cS} \Psi \\
%&= B^{1/2} (O^{\tilde{A}} \Lambda^{\tilde{A}}_{:,\cS} - O^{\tilde{A}}_{:,\cS} \Lambda^{\tilde{A}}_{\cS,\cS}) \Psi = 0.
%\end{align*}
\begin{proof}
	We have the Hessian of $\cL(X,Y)$ on $X$ with $Y=\cD(X)$ as
	\begin{align}
		H_X &= 2\sym\big( I_r \otimes (( BXX^\top - I_d ) A) + (X^\top A X) \otimes B + (AX) \boxtimes (BX) \big) \label{eqn:hessian}%\nonumber \\
		%&\hspace{0.2in} + (I_r - X^\top BX) \otimes A - I_r \otimes A X X^\top B - BX \boxtimes AX, \label{eqn:hessian}
	\end{align}
	where $\sym(M) = M + M^{\top}$, $\otimes$ is the Kronecker product, and for $U \in \RR^{d \times r}$ and $V \in \RR^{m \times k}$, $U \boxtimes V \in \RR^{dk \times mr}$ is defined as
	\begin{align*}
		U \boxtimes V = \left[ \begin{array}{cccc}
			U_{:,1} V_{:,1}^\top & U_{:,2} V_{:,1}^\top & \cdots & U_{:,r} V_{:,1}^\top \\
			U_{:,1} V_{:,2}^\top & U_{:,2} V_{:,2}^\top & \cdots & U_{:,r} V_{:,2}^\top \\
			\vdots & \vdots & \ddots & \vdots \\
			U_{:,1} V_{:,k}^\top & U_{:,2} V_{:,k}^\top & \cdots & U_{:,r} V_{:,k}^\top
		\end{array} \right].
	\end{align*}
	
	To determine whether a stationary point is an unstable stationary or a minimax global optimum, we consider its Hessian. We start with checking that $\cS = [r]$ corresponds to the global optimum, $X = B^{-1/2} O^{\tilde{A}}_{:,[r]} \Psi$. Without loss of generality, we set $\Psi = I_r$. We only need to check that for any vector $v = [v_1^{\top},\ldots,v_r^{\top}]^{\top} \in \RR^{nr}$ with $v_i \in \RR^n$ denoting the $i$-th block of $v$, which satisfies
	\begin{align*}
		v_i = c_{j_i} B^{-1/2} O^{\tilde{A}}_{:,{j_i}} \text{ for any $j_i \in [d]$ and a real constant $c_j$}
	\end{align*}
	such that $\norm{v}_2 = 1$, then we have $v^\top H_X v \geq 0$. The general case is only a linear combination of such $v$'s. Specifically, for $X = O^{\tilde{A}}_{:,[r]}$, we have
	\begin{align*}
		v^\top H_X v &= -v^\top \sym\left( I_r \otimes ((I_d - B XX^\top) A) - (X^\top A X) \otimes B - (AX) \boxtimes (BX) \right) v \\
		&= -v^\top \sym\Big( I_r \otimes ((I_d - B^{1/2} O^{\tilde{A}}_{:,[r]} O^{\tilde{A} \top}_{:,[r]} B^{-1/2}) A) - (O^{\tilde{A} \top}_{:,[r]} B^{-1/2} A B^{-1/2} O^{\tilde{A}}_{:,[r]} ) \otimes B\\
		&\hspace{2.0in} - (A B^{-1/2} O^{\tilde{A}}_{:,[r]}) \boxtimes (B^{1/2} O^{\tilde{A}}_{:,[r]}) \Big) v \\
		&= -v^\top \sym\Big( I_r \otimes ( B^{1/2} O^{\tilde{A}}_{:,[d]\backslash [r]} O^{\tilde{A} \top}_{:,[d]\backslash [r]} B^{-1/2} A ) - \Lambda^{\tilde{A}}_{:, [r]} \otimes B  - (B^{1/2} O^{\tilde{A}} \Lambda^{\tilde{A}}_{:, [r]}) \boxtimes (B^{1/2} O^{\tilde{A}}_{:,[r]}) \Big) v \\
		&= -2 \sum_{i=1}^r c_{j_i}^2 O^{\tilde{A} \top}_{:,j_i} O^{\tilde{A}}_{:,[d]\backslash [r]} \Lambda^{\tilde{A}}_{[d]\backslash [r],:} O^{\tilde{A} \top} O^{\tilde{A}}_{:,j_i} + 2 \sum_{i=1}^r c_{j_i}^2 \lambda^{\tilde{A}}_{i} + 2 \sum_{i=1}^{r} \sum_{k=1}^{r} c_{j_i} c_{j_k} e_{j_i}^{\top} \Lambda^{\tilde{A}}_{:, k} O^{\tilde{A} \top}_{:,i} O^{\tilde{A}}_{j_k} \\
		&\geq 0 + 2 \sum_{i=1}^r c_{j_i}^2 \lambda^{\tilde{A}}_{i} + 2 \sum_{i=1}^{r} \sum_{k=1}^{r} c_{j_i} c_{j_k} \lambda^{\tilde{A}}_{j_i} = 0,
	\end{align*}
	where the last inequality is obtained by taking $j_k \in [r]$, $i=j_k$, and $k=j_i$ in the last term, and the last equality is obtained by setting $c_{j_k} = -c_{j_i}$ when $j_i=k$, which implies that the restricted strongly convex property at $X$ holds.

	For any other $\cI \neq [r]$, we only need to show that the largest eigenvalue of $\nabla^2 \cL$ is positive and the smallest eigenvalue of $\nabla^2 \cL$ is negative, which implies that such a stationary point is unstable. Using the same construction as above, we have
	\begin{align*}
		&\lambda_{\min} (H_X) \leq -v^\top \sym \left( I_r \otimes ((I_d - B XX^\top) A) - (X^\top A X) \otimes B - (AX) \boxtimes (BX) \right) v\\
		&= -v^\top \sym \Big( I_r \otimes ( B^{1/2} O^{\tilde{A}}_{:,\cI^{\perp}} O^{\tilde{A} \top}_{:,\cI^{\perp}} B^{-1/2} A ) - \Lambda^{\tilde{A}}_{:, \cI} \otimes B - (B^{1/2} O^{\tilde{A}} \Lambda^{\tilde{A}}_{:, \cI}) \boxtimes (B^{1/2} O^{\tilde{A}}_{:,\cI}) \Big) v \\
		&= -2 \sum_{i \in \cI} c_{j_i}^2 O^{\tilde{A} \top}_{:,j_i} O^{\tilde{A}}_{:,\cI^{\perp}} \Lambda^{\tilde{A}}_{\cI^{\perp},:} O^{\tilde{A} \top} O^{\tilde{A}}_{:,j_i} + 2 \sum_{i \in \cI} c_{j_i}^2 \lambda^{\tilde{A}}_{i} + 2 \sum_{i \in \cI} \sum_{k \in \cI} c_{j_i} c_{j_k} e_{j_i}^{\top} \Lambda^{\tilde{A}}_{:, k} O^{\tilde{A} \top}_{:,i} O^{\tilde{A}}_{j_k} \\
		&\overset{(i)}{=}  2 c_{j_r}^2 ( \lambda^{\tilde{A}}_{\max \cI} - \lambda^{\tilde{A}}_{\min \cI^{\perp}}),
	\end{align*}
	where $(i)$ is from setting $c_{j_i}=0$ for all $j_i \in \cI^{\perp}$ except $j_r$, and $c_{j_r}= 1/\| B^{-1/2} O^{\tilde{A}}_{:,\min \cI^{\perp}}\|_2$.
	
	On the other hand, we have
	\begin{align*}
		&\lambda_{\max} (H_X) \geq v^\top H_X v \\
		&= -2 \sum_{i \in \cI} c_{j_i}^2 O^{\tilde{A} \top}_{:,j_i} O^{\tilde{A}}_{:,\cI^{\perp}} \Lambda^{\tilde{A}}_{\cI^{\perp},:} O^{\tilde{A} \top} O^{\tilde{A}}_{:,j_i} + 2 \sum_{i \in \cI} c_{j_i}^2 \lambda^{\tilde{A}}_{i} + 2 \sum_{i \in \cI} \sum_{k \in \cI} c_{j_i} c_{j_k} e_{j_i}^{\top} \Lambda^{\tilde{A}}_{:, k} O^{\tilde{A} \top}_{:,i} O^{\tilde{A}}_{j_k} \\
		&\overset{(i)}{=} 2 c_{j_1}^2\lambda^{\tilde{A}}_{\min \cI} + c_{j_1}^2\lambda^{\tilde{A}}_{\min \cI} = 4 c_{j_1}^2 \lambda^{\tilde{A}}_{\min \cI},
	\end{align*}
	where $(i)$ is from setting $c_{j_i}=0$ for all $j_i \in \cI$ except $j_1$, and $c_{j_1} = 1/\| B^{-1/2} O^{\tilde{A}}_{:,\min \cI}\|_2$. 
	
\end{proof}

%% file: Append-2.tex
\section{Singular case for $B$}\label{sec:singular}
When \textbf{$B$ is Singular}, we assume $\rk(B) = m <d$ and $\rk(A)=d$. %For $\rk(A)<d$, similar results can be obtained straightforwardly. 
Note that we require $m \geq r$; Otherwise, the feasible region of \eqref{eqn:original} becomes $\cT_{B} = \emptyset$. 
% Becasue $B$ is degenerate, it is natural to do the projection of $X$ to the column space of $B$. Then, the primal problem converts to the non-singular case. 

Before we proceed with our analysis, we first exclude an ill-defined case, where the objective function of \eqref{eqn:original} is unbounded from above. The following proposition shows the sufficient and necessary condition of the existence of the global optima of \eqref{eqn:original}.
\begin{proposition}\label{nooptimal}
Given a full rank symmetric matrix $A\in\RR^{d\times d}$ and a positive semidefinite matrix $B\in\RR^{d\times d}$, the optimal solution of \eqref{eqn:original} exists if and only if for all $v\in \textrm{Null}(B)$, one of the following two condition holds: (1) $v^\top A v<0$; (2)  $v^\top A v=0$ and $u^\top A v=0$, $\forall u\in \textrm{Col}(B)$.
%$Av \in \textrm{Null}(B)\backslash\textrm{span}(v)$.
\end{proposition}
\begin{proof}
We decompose $X=X_B+X_{B^\perp}$, where $X_B=[u_1,...,u_r]$ with $u_i\in \textrm{Col}(B)$ and each column of $X_{B^\perp}=[v_1,...,v_r]$ with $v_i\in \textrm{Null}(B)$. Note such decomposition is unique. Then \eqref{eqn:original} becomes
\begin{align}\label{obj}
\min -\sum_{i=1}^r(u_i^\top A u_i)-2\sum_{i=1}^r(u_i^\top A v_i)-\sum_{i=1}^r(v_i^\top A v_i)  \quad \textrm{s.t.} \quad X_B^\top B X_B=I_r.
\end{align}
If \eqref{obj} has an optimal solution, we have  $v^\top A v \leq 0$, for all $v\in \textrm{Null}(B) $; otherwise, fixing the feasible $X_B$, we use $X_B=[\lambda v,...,\lambda v]$ and increase $\lambda$, then there is no lower bound of objective value. Further, given a vector $v\in \textrm{Null}(B)$ with $v^\top A v= 0$, $u^\top A v=0$ must hold for all $u\in \textrm{Col}(B)$; otherwise, W.L.O.G, we assume that $u_1\in \textrm{Col}(B)$ $u_1^\top A v>0$, we can construct a feasible $X_B=\mu [u_1,...,u_r]$, where $\mu$ is a normalization constant such that $\mu^2 u_1^\top Bu_1=1.$ Then constructing $X_{B^{\perp}}=\lambda[v,0,...0].$, if we increase $\lambda$, there is no lower bound the objective value. Therefore, for a vector $v\in \textrm{Null}(B)$, either $v^\top A v=0$, or $u^\top A v=0$ and $v^\top A v=0$ hold.
\end{proof}

 Throughout our following analysis, we exclude the ill-defined case. 

The idea of characterizing all the equilibria is analogous to the nonsingular case, but much more involved.  
%Suppose that there exists a vector $v\in \textrm{Null}(B)$ such that $v^\top A v>0$. As a result, there is no global optimal solution of problem~\eqref{eqn:original}. Later, we will show that under such an ill-defined setting, all the equilibria are unstable, which means there is no global optima in the min-max problem~\eqref{eqn:minimax}. This is also consistent with the primal problem~\eqref{eqn:original}.
Since B is singular, we need to use general inverses. For notationally convenience, we use block matrices in our analysis. We consider the eigenvalue decomposition of $B$ as follows:
\begin{align*}
	B &= \underbrace{\left[\begin{array}{cc}
		O_{11}^{B} & O_{12}^{B}\\
		O_{21}^{B} & O_{22}^{B}
	\end{array}\right]}_{O^B}\underbrace{ \left[\begin{array}{cc}
	\Lambda^{B}_{11} & 0\\
	0 & 0
\end{array}\right]}_{\Lambda^B} \underbrace{\left[\begin{array}{cc}
O^{B \top}_{11} & O^{B \top}_{21}\\
O^{B \top}_{12} & O^{B \top}_{22}
\end{array}\right]}_{O^{B\top}},
\end{align*}
where $O_{11}^{B}\in\RR^{m \times m}$, $O_{22}^{B}\in \RR^{(d-m)\times (d-m)}$, and $\Lambda^{B}_{11}=\diag(\lambda_1,...,\lambda_m)$ with $\lambda_1\geq\cdots \geq\lambda_m>0$ . We then left multiply $O^{B\top}$ and right multiply $O^B$ to $A$:
\begin{align*}
 O^{B \top} A O^{B}=: W = \left[\begin{array}{cc}
	W_{11} & W_{12}\\
	W_{21} & W_{22}
\end{array}\right],
\end{align*}
where $W_{11}\in \RR^{m \times m},  W_{22} \in \RR^{(d-m) \times (d-m)}$. Here, we assume $W_{22}$ is nonsingular (guaranteed in the well-defined case). Then we construct a general inverse of $\Lambda^B$. Specifically, given an arbitrary positive definite matrix $P\in\RR^{(d-m)\times (d-m)}$, we define $\Lambda^{B \dagger}(P)$ as
\begin{align*}
\Lambda^{B \dagger}(P):=\left[\begin{array}{cc}
(\Lambda^{B}_{11})^{-1}& 0\\
0 & P
\end{array}\right].
\end{align*}
 Note $\Lambda^{B \dagger}(P)$ is invertible and depends on $P$. Recall the primal variable $X$ at the equilibrium of $\cL(X,Y)$ satisfies 
\begin{align}\label{KKT}
AX=BX\cdot X^\top A X~~\textrm{and}~~ X^\top BX=I_r.
\end{align}
For notational simplicity, we define
\begin{align}\label{change}
 V(P) :=\left(\Lambda^{B \dagger}(P)\right)^{-\frac{1}{2}} O^{B\top}\hspace{-0.1cm} \left[\begin{array}{c}
		X_1 \\
		X_2
	\end{array}\right]= \left[\hspace{-0.1cm} \begin{array}{c}
		V_1 \\
		V_2(P)
	\end{array}\hspace{-0.1cm} \right],
\end{align} 
where $V_1,~X_1\in \RR^{m \times r},$ and $V_2(P),~X_2\in \RR^{(d-m)\times r}.$ Note that $V_1$ does not depend on $P$. From~\eqref{change} we have 
\begin{align}\label{invert}
 \left[\begin{array}{c}
		X_1 \\
		X_2
	\end{array}\right]=O^B \left(\Lambda^{B \dagger}(P)\right)^{\frac{1}{2}} \left[\hspace{-0.1cm} \begin{array}{c}
		V_1 \\
		V_2(P)
	\end{array}\hspace{-0.1cm} \right].
\end{align}
Combining \eqref{invert} and \eqref{KKT} we get the following equation system:
 \begin{subequations}
	\begin{empheq}[left={\empheqlbrace\,}]{align}
		&\tilde{A}(P)V(P)=\left[\begin{array}{c}V_1\\0 \end{array}\right]V(P)^\top \tilde{A}(P) V(P), \label{eqn:1} \\
		&V(P)^\top \diag(I_m,0)V(P)=I_r,\label{eqn:2}
	\end{empheq} \label{eqn:equiva}
\end{subequations}

\noindent where $\tilde{A}(P)=(\Lambda^{B\dagger}(P))^{\frac{1}{2}} W (\Lambda^{B\dagger}(P))^{\frac{1}{2}}$. The invertibility of $\Lambda^{B \dagger}(P)$ ensures that solving \eqref{KKT} is equivalent to doing the transformation \eqref{invert} to the solution of \eqref{eqn:equiva}. We then denote $$\hat{A}=(\Lambda^{B}_{11})^{-\frac{1}{2}} \left(W_{11}-W_{12}W^{-1}_{22}W_{21}\right)(\Lambda^{B}_{11})^{-\frac{1}{2}}$$ and consider its eigenvalue decomposition as $\hat{A}=O^{\hat{A}}\Lambda^{\hat{A}}O^{\hat{A}\top}$. The following theorem characterizes all the equilibria of $\cL(X,Y)$ with a singular $B$.
%By \eqref{KKT} and general inverses of $B$, we get the following theorem, which characterizes the generation of all the equilibria.
\begin{theorem}\label{thm:singular}
Given a full rank symmetric matrix $A\in\RR^{d\times d}$ and a positive semidefinite matrix $B\in \RR^{d\times d}$ with $\rk(B)=m<d$, satisfying the well-defined condition in Proposition~\ref{nooptimal},$(X,\cD(X))$ is an equilibrium of $\cL(X,Y)$ if and only if $X$ can be represented as
\begin{align*}
X=O^B\left[
\begin{array}{c} 
(\Lambda_{11}^B)^{-\frac{1}{2}}\cdot O_{:,\cI}^{\hat{A}}\\ 
-W_{22}^{-1}W_{12}^\top (\Lambda^B_{11})^{-\frac{1}{2}}O_{:,\cI}^{\hat{A}}  
\end{array}
\right]
\cdot \Psi,
\end{align*}
where $\Psi\in\cG$ and $\cI\in\cX_m$ is the column index set.
\end{theorem}

\begin{proof}
By definition, we have 
\begin{align}\label{eqn:begin}
\left\{
\begin{array}{c}
AX=BX\cdot Y \\
X^\top BX=I_r
\end{array}
\right.\Longrightarrow 
\left\{
\begin{array}{c}
AX=BX\cdot X^\top A X \\
X^\top BX=I_r
\end{array}
\right.,
\end{align}
We define $ V(P) :=\left(\Lambda^{B \dagger}(P)\right)^{-\frac{1}{2}} O^{B\top}\hspace{-0.1cm} \left[\begin{array}{c}
		X_1 \\
		X_2
	\end{array}\right]= \left[\hspace{-0.1cm} \begin{array}{c}
		V_1 \\
		V_2(P)
	\end{array}\hspace{-0.1cm} \right],$ where $V_1,~X_1\in \RR^{m \times r},$ and $V_2(P),~X_2\in \RR^{(d-m)\times r}.$ Note that $V_1$ does not depend on $P$. By \eqref{eqn:begin} and replacing $I_d$ with $O^B O^{B\top}$ and $\Lambda^{B\dagger}(P)^{\frac{1}{2}}\Lambda^{B\dagger}(P)^{-\frac{1}{2}}$, we have 
	\begin{subequations}
%\vspace{-0.2in}
	\begin{empheq}[left={\empheqlbrace\,}]{align}
		&\tilde{A}(P)V(P)=\left[\begin{array}{c}V_1\\0 \end{array}\right]V(P)^\top \tilde{A}(P) V(P), \label{eqn:5} \\
		&V(P)^\top \diag(I_m,0)V(P)=I_r,\label{eqn:6}
	\end{empheq} \label{eqn:equiva1}
\end{subequations}
where $\tilde{A}(P)=(\Lambda^{B\dagger}(P))^{\frac{1}{2}} W (\Lambda^{B\dagger}(P))^{\frac{1}{2}}$. Simplifying \eqref{eqn:5}, we obtain
\begin{align*}
\left\{
\begin{array}{c}
W_{22}^{-1}W_{21}(\Lambda^B_{11})^{-\frac{1}{2}}V_1=P^{\frac{1}{2}}V_2(P)	 \\
V_1 V_1^\top (\Lambda^B_{11})^{-\frac{1}{2}} (W_{11}-W_{12}W_{22}^{-1}W_{21})    (\Lambda^B_{11})^{-\frac{1}{2}}
\end{array}
\right.
\end{align*}
Let $\hat{A}=(\Lambda^{B}_{11})^{-\frac{1}{2}} \left(W_{11}-W_{12}W^{-1}_{22}W_{21}\right)(\Lambda^{B}_{11})^{-\frac{1}{2}}$. Then, by \eqref{eqn:equiva1}, we obtain the following equations:
%\begin{align}\label{limit}
%(I_m-V_1V_1^\top)(\Lambda^{B}_{11})^{\frac{1}{2}} W_{11}(\Lambda^{B}_{11})^{\frac{1}{2}}V_1=0;
%\end{align}
\begin{subequations}
	\begin{empheq}[left={\empheqlbrace\,}]{align}
		&\hat{A} V_1=V_1V_1^\top \hat{A} V_1 , \label{eqn:3} \\
		&V_1^\top V_1=I_r,\label{eqn:4}
	\end{empheq} \label{eqn:new}
\end{subequations}
%Simplifying \eqref{eqn:eigdecomp_generalB} algebraically, we have
%\begin{align}\label{eqn:V1} 
% V_1 \Lambda^{V}_{\cS',\cS'}&=-(\Lambda^{B}_{11})^{-\frac{1}{2}}   W_{12} W_{22}^{-1} W_{21} (\Lambda^{B}_{11})^{-\frac{1}{2}} V_1\notag\\
%  &~~~~~+(\Lambda^{B}_{11})^{-\frac{1}{2}} W_{11} (\Lambda^{B}_{11})^{-\frac{1}{2}} V_1.
%\end{align}
%Then, we can further calculate the $X_1$ and $X_2$.
%\begin{align}\label{eqn:X1} 
%X_1 &= \hspace{-0.15cm}-O^{B}_{11} O^{B \top}_{21} (O^{B \top}_{22})^{-1} O^{B \top}_{12}\hspace{-0.15cm}+O^{B}_{11}(\Lambda^{B}_{11})^{-\frac{1}{2}}  V_1 ,
%\end{align}
%\vspace{-0.35in}
%\begin{align} 
%X_2& =(O^{B \top}_{22})^{-1} W_{22}^{-1} W_{21} (\Lambda^{B}_{11})^{-\frac{1}{2}} V_1 \notag\\
%&~~~~~~~~~~~~~~~~- (O^{B \top}_{22})^{-1} O^{B \top}_{12} X_1.\label{eqn:X2} 
%\end{align}
Note \eqref{eqn:new} are the KKT conditions of the following problem:
\begin{align}\label{new}
V_1^*=\argmin_{V_1\in\RR^{m\times r}}  -\tr(V_1^\top \hat{A} V_1) \quad \textrm{s.t.  } ~~V_1^\top V_1=I_r.
\end{align} 
Because \eqref{new} is not a degenerate case, Theorem~\ref{thm:stationary} can be directly applied to \eqref{new}. Then, we get the stable equilibria and unstable equilibria of \eqref{new}. Specifically, denote the eigenvalue decomposition of $\hat{A}$ as $\hat{A}=O^{\hat{A}}\Lambda^{\hat{A}}O^{\hat{A}\top}$. Then we know the equilibrium of \eqref{eqn:new} can be represented as $V_1=O^{\hat{A}}_{:,\cI}\cdot \Psi$, where $\cI\in\Big\{\{i_1,...,i_r\}:\{i_1,...,i_r\}\subseteq [m]\Big\}$ and $\Psi\in\cG$. Then, we know the primal variable $X$ at an equilibrium of $\cL(X,Y)$ satisfies 
\begin{align*}
X=O^B\left[
\begin{array}{c} 
(\Lambda_{11}^B)^{-\frac{1}{2}}\cdot O_{:,\cI}^{\hat{A}}\\ 
-W_{22}^{-1}W_{12}^\top (\Lambda^B_{11})^{-\frac{1}{2}}O_{:,\cI}^{\hat{A}}  
\end{array}
\right]
\cdot \Psi,
\end{align*}
where $O_{:,\cI}^{\hat{A}}$ is an equilibrium for the Lagrangian function of \eqref{new}.
 \end{proof}

 Theorem~\ref{thm:singular} implies that for the well-defined degenerated case, there are only $\binom{m}{r}$ equilibria unique in the sense of invariant group, since $B$ is rank deficient.

\section{Proofs for the Convergence Rate of Algorithm.}

\subsection{Proof of Lemma~\ref{lem:ode}}\label{pf:lem:ode}
\begin{proof}
	Denote $k=\lfloor \frac{t}{\eta} \rfloor$, $\Delta(t)=w^{(\eta)}(t+\eta)-w^{(\eta)}(t)$,  $\Delta_i$ as the $i$-th component of $\Delta$. For notational simplicity, we may drop $(t)$ if it is clear from the context. By the definition of $w^{\eta}(t)$, we have %in \eqref{eqn:weta}, we have 
	\begin{align}\label{equ:basic}
	\frac{1}{\eta} \EE \left(\Delta(t) \Big| w^{(\eta)}(t) \right)&  =\frac{1}{\eta}\EE\left(W^{(k+1)}-W^{(k)}\big| W^{(k)}\right)\notag\\
	& =\frac{1}{\eta}\EE\left[\eta \left(\Lambda^B-(\Lambda^B)^{\frac{1}{2}}\hat{\Lambda}_B^{(k)}(\Lambda^B)^{-\frac{1}{2}}W^{(k)}W^{(k)\top}\right)\cdot\tilde{\Lambda}^{(k)}W^{(k)}\big| W^{(k)}\right]\notag\\
	& =\Lambda^Aw^{(\eta)}(t) -\left( w^{(\eta)}(t) \right)^\top\left(\Lambda^B\right)^{-\frac{1}{2}}\Lambda^A\left(\Lambda^B\right)^{-\frac{1}{2}}w^{(\eta)}(t)\Lambda^B  w^{(\eta)}(t).
	\end{align}
	Similarly, we calculate the infinitesimal conditional expectation of $v_{i,1}=\frac{(w^{(\eta)}_i)^{\mu_1}}{(w^{(\eta)}_1)^{\mu_i}}$ as
	\begin{align*}
	&\textstyle \frac{1}{\eta} \EE \left(\frac{\left(w^{(\eta)}_i\right)^{\mu_1}}{\left(w^{(\eta)}_1\right)^{\mu_i}}(t+\eta)-\frac{\left(w^{(\eta)}_i\right)^{\mu_1}}{\left(w^{(\eta)}_1\right)^{\mu_i}}(t) \Bigg|\frac{\left(w^{(\eta)}_i\right)^{\mu_1}}{\left(w^{(\eta)}_1\right)^{\mu_i}}(t)\right)\\
	\textstyle = &\frac{1}{\eta} \EE \left(\frac{\left(w^{(\eta)}_i(t)+\Delta_i\right)^{\mu_1}}{\left(w^{(\eta)}_1(t)+\Delta_1\right)^{\mu_i}}-\frac{\left(w^{(\eta)}_i(t)\right)^{\mu_1}}{\left(w^{(\eta)}_1(t)\right)^{\mu_i}} \Bigg|\frac{\left(w^{(\eta)}_i(t)\right)^{\mu_1}}{\left(w^{(\eta)}_1(t)\right)^{\mu_i}}\right)\\
%	\textstyle  = & \frac{1}{\eta} \EE \left(\frac{\left(w^{(\eta)}_i(t)\right)^{\mu_1}}{\left(w^{(\eta)}_1(t)\right)^{\mu_i}}\cdot \frac{(1+\frac{\Delta_i}{w^{(\eta)}_i(t)})^{\mu_1}}{(1+\frac{\Delta_1}{w^{(\eta)}_1(t)})^{\mu_i}}-\frac{\left(w^{(\eta)}_i(t)\right)^{\mu_1}}{\left(w^{(\eta)}_1(t)\right)^{\mu_i}} \Big|\frac{\left(w^{(\eta)}_i(t)\right)^{\mu_1}}{\left(w^{(\eta)}_1(t)\right)^{\mu_i}}\right)\\
	\textstyle  = &\frac{1}{\eta} \frac{\left(w^{(\eta)}_i(t)\right)^{\mu_1}}{\left(w^{(\eta)}_1(t)\right)^{\mu_i}} \EE \left( [1+\mu_1 \frac{\Delta_i}{w^{(\eta)}_i}+\cO(\eta^2)] \cdot [1-\mu_i \frac{\Delta_1}{w^{(\eta)}_1}+\cO(\eta^2)]-1 \Big| w^{(\eta)}(t) \right)\\
	\textstyle  =& \frac{1}{\eta} \frac{\left(w^{(\eta)}_i(t)\right)^{\mu_1}}{\left(w^{(\eta)}_1(t)\right)^{\mu_i}} \left(\frac{\mu_1}{w^{(\eta)}_i} \EE (\Delta_i\big| w^{(\eta)}(t)) -\frac{\mu_i}{w^{(\eta)}_1} \EE( \Delta_1\big| w^{(\eta)}(t)) \right)+\cO(\eta)\\
	\textstyle  = &\frac{\left(w^{(\eta)}_i(t)\right)^{\mu_1}}{\left(w^{(\eta)}_1(t)\right)^{\mu_i}} \Big[\frac{\mu_1}{w^{(\eta)}_i} \left( -\sum_{k=1}^d \frac{\lambda_k}{\mu_k}\left(w^{(\eta)}_k\right)^2 \mu_i w^{(\eta)}_i + \lambda_iw^{(\eta)}_i \right)- \\
	\textstyle&~~~~~~~~~~~~~~~~~ \frac{\mu_i}{w^{(\eta)}_1} \left( -\sum_{k=1}^d \frac{\lambda_k}{\mu_k}\left(w^{(\eta)}_k\right)^2 \mu_1 w^{(\eta)}_1 + \lambda_1w^{(\eta)}_1 \right)\Big]+\cO(\eta)\\
	&\textstyle  = \frac{\left(w^{(\eta)}_i(t)\right)^{\mu_1}}{\left(w^{(\eta)}_1(t)\right)^{\mu_i}} \mu_1\mu_i \left( \beta_i-\beta_1\right)+\cO(\eta),
	\end{align*}
	where the third equality holds because of the Taylor expansion, the fourth holds for $\Delta$ is order of $\cO(\eta)$ and the last equality holds due to \eqref{equ:basic}.
	Then, we calculate the infinitesimal conditional variance. From the update of $W$ in \eqref{equivalent}, if $t\in [0,~T]$ with a finite $T$, then $w^{(\eta)}(t)$ is bounded with probability $1$. Denote $\norm{w^{(\eta)}(t)}_2^2 \leq D < \infty$. Then we have
	\begin{align*}
	&\textstyle \frac{1}{\eta} \EE \left[\left(\frac{\left(w^{(\eta)}_i\right)^{\mu_1}}{\left(w^{(\eta)}_1\right)^{\mu_i}}(t+\eta)-\frac{\left(w^{(\eta)}_i\right)^{\mu_1}}{\left(w^{(\eta)}_1\right)^{\mu_i}}(t) \right)^2\Bigg|\frac{\left(w^{(\eta)}_i\right)^{\mu_1}}{\left(w^{(\eta)}_1\right)^{\mu_i}}(t)\right] \\
	\textstyle=&  \frac{1}{\eta}\frac{\left(w^{(\eta)}_i(t)\right)^{2\mu_1}}{\left(w^{(\eta)}_1(t)\right)^{2\mu_i}}\EE\left[\left( \mu_1 \frac{\Delta_i}{w^{(\eta)}_i}-\mu_i\frac{\Delta_1}{w^{(\eta)}_1}\right)^2\bigg|w^{(\eta)}(t) \right]+\cO(\eta^2)\\
	\textstyle  \leq &\frac{2}{\eta}\frac{\left(w^{(\eta)}_i(t)\right)^{2\mu_1}}{\left(w^{(\eta)}_1(t)\right)^{2\mu_i}} \EE\left[ \left(\frac{\mu_1}{w^{(\eta)}_i}\right)^2 \Delta_i^2+\left(\frac{\mu_i}{w^{(\eta)}_1}\right)^2 \Delta_1^2\big| w^{(\eta)}(t)\right]+\cO(\eta^2)\\
	%&\textstyle  \leq \frac{2}{\eta}\frac{\left(w^{(\eta)}_i(t)\right)^{2\mu_1}}{\left(w^{(\eta)}_1(t)\right)^{2\mu_i}} \cdot \left( \left(\frac{\mu_1}{w^{(\eta)}_i}\right)^2+\left(\frac{\mu_i}{w^{(\eta)}_1}\right)^2 \right)\EE \left[\Delta^\top \Delta\big| w^{(\eta)}(t)\right]+\cO(\eta^2)\\
	%&\textstyle  \leq \frac{2}{\eta}\frac{\left(w^{(\eta)}_i(t)\right)^{2\mu_1}}{\left(w^{(\eta)}_1(t)\right)^{2\mu_i}} \cdot \left( \left(\frac{\mu_1}{w^{(\eta)}_i}\right)^2+\left(\frac{\mu_i}{w^{(\eta)}_1}\right)^2 \right) \eta^2\cdot \left( C_{4,-4}\cdot D^3+ 2\times C_{3,-2} \cdot D^2 + C_{2,0} \cdot D\right)+\cO(\eta^2)\\
	\textstyle  \leq &4\eta\frac{\left(w^{(\eta)}_i(t)\right)^{2\mu_1}}{\left(w^{(\eta)}_1(t)\right)^{2\mu_i}} \EE\Big[ \frac{\left((w^{(\eta)})^\top \tilde{\Lambda}^{(k)} w^{(\eta)}\right)^2\left(e_i^\top (\hat{\Lambda}_B^{(k)})\left(\Lambda^B\right)^{-\frac{1}{2}}w^{(\eta)}\right)^2+\mu_i\left(e_i\tilde{\Lambda}^{(k)}w^{(\eta)}\right)^2}{(w^{(\eta)}_i)^2} \mu_i \mu_1^2\\
	&\textstyle  \hspace{0.2in}+\frac{\left((w^{(\eta)})^\top \tilde{\Lambda}^{(k)} w^{(\eta)}\right)^2\left(e_1^\top (\hat{\Lambda}_B^{(k)})\left(\Lambda^B\right)^{-\frac{1}{2}}w^{(\eta)}\right)^2+\mu_1\left(e_1\tilde{\Lambda}^{(k)}w^{(\eta)}\right)^2}{(w^{(\eta)}_1)^2}   \mu_1\mu_i^2   \big| w^{(\eta)}(t)\Big]+\cO(\eta^2)\\
	\textstyle \leq &4\eta^{2\delta} \left(C \frac{C_0C_1}{\mu_{\min}^2}D^3\mu_i\mu_1^2+\mu_i^2\mu_1^2\frac{C_0}{\mu_{\min}}D\right)+\cO(\eta)\\
	\textstyle  =& \cO(\eta^{2\delta})\xrightarrow{\eta\rightarrow 0}0,
	\end{align*}
	where the second inequality holds because of the mean inequality and the last inequality is from the independence of $A^{(k)}$ and $B^{(k)}$, $(w^{(\eta)})^\top \tilde{\Lambda}w^{(\eta)}\leq \norm{\tilde{\Lambda}}_2(w^{(\eta)})^\top w^{(\eta)}\leq \frac{\norm{A^{(k)}}_2}{\mu_{\min}}D$, since $\tilde{\Lambda}$ is symmetric, and $C=\frac{\left(w^{(\eta)}_i(t)\right)^{2\mu_1-2}}{\left(w^{(\eta)}_1(t)\right)^{2\mu_i}}$.
	By Section 4 of Chapter 7 in \cite{ethier2009markov}, we have that when $t\in [0,T]$, as $\eta \rightarrow 0$, $\frac{(w^{(\eta)}_i)^{\mu_1}}{(w^{(\eta)}_1)^{\mu_i}}$ weakly converges to the solution of~\eqref{equ:ode} if they have the same initial solutions. Then, let $T\rightarrow \infty$, we know the convergence of $\frac{(w^{(\eta)}_i)^{\mu_1}}{(w^{(\eta)}_1)^{\mu_i}}$ holds at any time $t$. 
	Note we can replace $1$ by $j$, where $j\neq i$, and the proof still holds.
       
        Moreover, using the same techniques, we can show that for all $i\in [d]$, $w_i^{(\eta)}$ converges to the solution of the following equation: 
	\begin{align}\label{conv}
	\frac{dw_i}{dt}=\mu_i(\beta_i-\sum_{j=1}^d \beta_j w_j^2 )w_i.
	\end{align}
	 %does not diverge. Otherwise, $\mu_1(\beta_1-\sum_{j=1}^d \beta_j w_j^2 )\leq 1$ leads to decrease $w_1$.
	%By \eqref{equ:ode} with $j=1$, we know, $\frac{(w^{(\eta)}_1)^{\mu_i}}{(w^{(\eta)}_i)^{\mu_1}}$ increases fast. Then by  \eqref{equ:ode} with $i=1$, $\frac{(w^{(\eta)}_j)^{\mu_1}}{(w^{(\eta)}_1)^{\mu_j}}$ converges to $0~\forall i\neq 1$.  Therefore, $w^{(\eta)}_1$ 
 Note that if any $w_i> 1$, $\mu_i(\beta_i-\sum_{j=1}^d \beta_j w_j^2 )w_i<0$, and if $\sum_{j=1}^d  w_j^2<1$, $\mu_1(\beta_1-\sum_{j=1}^d \beta_j w_j^2 )w_1>0$, which means that $w_1$ will increase. This further indicates that $w_1$ converges to $1$, while $w_i$ converges to 0 for all $i\neq 1$. This shows our algorithm converges to the neighbor of the global optima. 
\end{proof}

\subsection{Proof of Lemma~\ref{lem:sadd}}\label{pf:lem:sadd}
\begin{proof}
We prove this by contradiction. Assume the conclusion does not hold, that is there exists a constant $C>0,$ such that for any  $\eta'>0$ we have $$\sup_{\eta\leq \eta'}\PP(\sup_t |z^{(\eta)}_i(t)|\leq C)=1.$$ That implies there exists a sequence $\{\eta_n\}_{n=1}^\infty$ converging to $0$ such that 
\begin{align}\label{eq_contra}
\lim_{n\rightarrow\infty}\PP(\sup_t |z^{(\eta_n)}_i(t)|\leq C)= 1.
\end{align}
Thus, condition (i) in Theorem 2.4 \citep{nowakowski2013multi} holds. We next check the second condition. When $\sup_t |z^{(\eta_n)}_i(t)|\leq C$ holds, Assumption \ref{assmp} yields that $Z^{(\eta_n,k+1)}_i-z^{(\eta_n,k)}_i=C'\eta_n,$ where $C'$ is some constant. Thus, for any $t,\epsilon>0,$ we have $$|z_i^{(\eta_n)}(t)-Z_i^{(\eta_n)}(t+\epsilon)|=\frac{\epsilon}{\eta} C'\eta=C'\epsilon.$$ Thus, condition (ii) in Theorem Theorem 2.4 \citep{nowakowski2013multi} holds. Then we have  $\{Z_i^{(\eta_n)}(\cdot)\}_n$ is tight and thus converges weakly. We then calculate the infinitesimal conditional expectation
	\begin{align*}
	\frac{d}{dt} \EE(z^{(\eta)}_j(t)) & = \frac{1}{\eta}\EE\left(z^{(\eta)}_j(t+\eta)-z^{(\eta)}_j(t)\big|z^{(\eta)}_j(t)\right)  =  \eta^{-\frac{3}{2}}\EE\left(w^{(\eta)}_j(t+\eta)-w^{(\eta)}_j(t)\big|w^{(\eta)}_j(t)\right)\\
	& = - \eta^{-\frac{1}{2}}\left[\left( w^{(\eta)}(t) \right)^\top\left(\Lambda^B\right)^{-\frac{1}{2}}(\Lambda^A)\left(\Lambda^B\right)^{-\frac{1}{2}}w^{(\eta)}(t)\cdot\left(\Lambda^B\right) w^{(\eta)}(t)-(\Lambda^A)w^{(\eta)}(t)\right]_j\\
	& = \lambda_i z_j -\beta_i\mu_j z_j +\cO(\eta^{1-2\delta}).
	\end{align*}
	The last equality holds due to the fact that our initial point is near the saddle point $w_i^{(\eta)}(t)\approx e_i$ and $|w_j^{(\eta)}(t)|\leq C\eta^{\frac{1}{2}+\delta} $ . Next, we turn to the infinitesimal conditional variance,
	\begin{align*}
	&\textstyle \frac{1}{\eta}\EE\left[\left(z^{(\eta)}_j(t+\eta)-z^{(\eta)}_j(t)\right)^2\big|z^{(\eta)}_j(t)\right]\\
	\textstyle = &\EE\left[\left(e_j^\top \left(\left(\Lambda^B\right)^{\frac{1}{2}}\hat{\Lambda}_B^{(k)}\left(\Lambda^B\right)^{-\frac{1}{2}} w^{(k)}w^{(l)\top}- \Lambda^B\right)\cdot\tilde{\Lambda}w^{(k)} \right)^2\big|w^{(k)}\right]\\
	\textstyle = &\EE\left[\left(\left(\hat{\Lambda}_B^{(k)}\right)_{j,i}\cdot \sqrt{\mu_j/\mu_i}\cdot \tilde{\Lambda}_{i,i}-\mu_j\tilde{\Lambda}_{j,i}\right)^2 \right]+\cO(\eta^{3-6\delta})\\ =& G_{j,i}+\cO(\eta^{3-6\delta})
	\textstyle\leq 2\left(\frac{\mu_1}{\mu_j}\cdot C_0 \cdot C_1+\mu_i^2\cdot C_1\right).
	\end{align*}
	Then, we get the limit stochastic differential equation,
\begin{equation*}
	d z_j(t) =\left( -\beta_j\mu_i \cdot z_i +\lambda_i z_i \right)dt+\sqrt{G_{j,i}}d B(t)~~\textrm{for $j\in[d]\backslash\{i\}$.}
\end{equation*}
	Therefore, $\{Z_i^{(\eta_n)}(\cdot)\}$ converges weakly to a 
% \begin{align}\label{SDE_1}
%  dU^i=\frac{\lambda_i-\lambda_j}{1-\mu}U^i dt+\frac{\alpha_{i,j}}{1-\mu}dB_t.
%  \end{align} 
 process defined by the equation above, which is an unstable O-U process with mean $0$ and exploding variance. Thus, for any $\tau$, there exist a time $t'$, such that
$$\PP(|z_i(t')|\geq C)\geq 2\tau.$$
Since $\{z_i^{(\eta_n)}\}_n$ converges weakly to $z_i,$ thus $\{z_i^{(\eta_n)}(t')\}_n$ converges in distribution to $Z_(t').$ This implies that there exists an $N>0$, such that for any $n>N$ 
$$|{\PP(|z_i(T)|\geq C)-\PP(|z_i^{(\eta_n)}(T)|\geq C)}|\leq {\tau}.$$
Then we find a $t'$ such that 
$$\PP(|z_i^{(\eta_n)}(t')|\geq C)\geq {\tau}, \forall n>N, $$ 
or equivalently $$\PP(|z_i^{(\eta_n)}(t')|\leq C)< 1-{\tau},  \forall n>N.$$ 
Since $\left\{\omega\big|\sup_t |z_i^{(\eta_n)}(t)(\omega)|\leq C\right\}\subset\left\{\omega\big||z_i^{(\eta_n)}(\tau')(\omega)<C\right\},$ we have 
$$\PP(\sup_t |w_i^{(\eta_n)}(t)|\leq C\sqrt{\eta_n})=\PP(\sup_t |z_i^{(\eta_n)}(t)|\leq C)\leq 1-{\delta}, \forall n>N, $$ 
which leads to a contradiction with \eqref{eq_contra}. Our assumption does not hold.

\end{proof}

\subsection{Proof of Lemma~\ref{lem:sde}}\label{pf:lem:sde}
\begin{proof}
	Suppose the initial is near the stable equilibria, i.e., $|w^{(\eta)}_1(0)-1|\leq C \eta^{\frac{1}{2}+\delta}$ and $|w^{(\eta)}_j(0)|\leq C \eta^{\frac{1}{2}+\delta}$ for all $j\neq 1$.First we show that $\norm{w^{(\eta)}(t)}_2 \rightarrow 1$ as $t \rightarrow \infty$. With update~\eqref{equivalent}, we show $w^{(\eta)\top} w^{(\eta)}(t)$ weakly converges to the following ODE by a similar proof in Lemma~\ref{lem:ode}: 
\begin{align*}
	\frac{d}{dt}\EE \left(w^{(\eta)\top} w^{(\eta)}(t) \right) & = - w^{(\eta)\top}\left(\Lambda^B\right)^{-\frac{1}{2}}(\Lambda^A)\left(\Lambda^B\right)^{-\frac{1}{2}}w^{(\eta)}\cdot w^{(\eta)\top} \Lambda^B w^{(\eta)}+w^{(\eta)\top}(\Lambda^A)w^{(\eta)}+\cO(\eta)\\
	& = -\lambda_1 \left( \norm{w^{(\eta)}}_2^4-\norm{w^{(\eta)}}_2^2\right)+\cO(\eta^{1-2\delta}),
	\end{align*}
	Similarly, we can bound the infinitesimal conditional variance. Therefore, the norm of $w$ weakly converges to the following ODE:
	\begin{align*}
	d x = -\lambda_1 \left( x^2-x\right) dt.
	\end{align*}
	The solution of the above ODE is
	\begin{align*}
	x=\left\{
	\begin{array}{cc}
	\frac{1}{1-\exp(-\lambda_1 t+C)} & \quad\textrm{if $x>1$}\\
	\frac{1}{1+\exp(-\lambda_1 t+C)} & \quad\textrm{if $x<1$}\\
	1 & \quad \textrm{if $x=1$}
	\end{array}
	\right..
	\end{align*}
	This implies that $\| w^{(\eta)}(t)\|_2$ converges to 1 as $t \rightarrow \infty$. Then we calculate the infinitesimal conditional expectation for $i\neq 1$
	\begin{align*}
	\frac{d}{dt} \EE(z^{(\eta)}_i(t)) & = \frac{1}{\eta}\EE\left(z^{(\eta)}_i(t+\eta)-z^{(\eta)}_i(t)\big|z^{(\eta)}_i(t)\right)  =  \eta^{-\frac{3}{2}}\EE\left(w^{(\eta)}_i(t+\eta)-w^{(\eta)}_i(t)\big|w^{(\eta)}_i(t)\right)\\
	& = - \eta^{-\frac{1}{2}}\left[\left( w^{(\eta)}(t) \right)^\top\left(\Lambda^B\right)^{-\frac{1}{2}}(\Lambda^A)\left(\Lambda^B\right)^{-\frac{1}{2}}w^{(\eta)}(t)\cdot\left(\Lambda^B\right) w^{(\eta)}(t)-(\Lambda^A)w^{(\eta)}(t)\right]_i\\
	& = \lambda_i z_i -\beta_1\mu_i z_i +\cO(\eta^{1-2\delta}).
	\end{align*}
 The last equality is from the fact that our initial point is near an optimum. Next, we turn to the infinitesimal conditional variance,
	\begin{align*}
	&\textstyle \frac{1}{\eta}\EE\left[\left(z^{(\eta)}_i(t+\eta)-z^{(\eta)}_i(t)\right)^2\big|z^{(\eta)}_i(t)\right] = \EE\left[\left(e_i^\top \left(\left(\Lambda^B\right)^{\frac{1}{2}}\hat{\Lambda}_B^{(k)}\left(\Lambda^B\right)^{-\frac{1}{2}} w^{(k)}w^{(l)\top}- \Lambda^B\right)\cdot\tilde{\Lambda}w^{(k)} \right)^2\big|w^{(k)}\right]\\
	&\textstyle = \EE\left[\left(\left(\hat{\Lambda}_B^{(k)}\right)_{i,1}\cdot \sqrt{\mu_i/\mu_1} \cdot \tilde{\Lambda}_{1,1}-\mu_i\tilde{\Lambda}_{i,1}\right)^2 \right]+\cO(\eta^{3-6\delta}) \\
	&\textstyle= G_{i,1}+\cO(\eta^{3-6\delta}) \leq 2\left(\frac{\mu_i}{\mu_1}\cdot C_0 \cdot C_1+\mu_i^2\cdot C_1\right).
	\end{align*}
	By Section 4 of Chapter 7 in \cite{ethier2009markov}, we have that the algorithm converges to the solution of~\eqref{equ:sde} if it is already near our optimal solution.
\end{proof}

\subsection{Proof of Theorem~\ref{thm:num-iter-step-size}}\label{pf:thm:num-iter-step-size}
\begin{proof}

Assume the initial is near a saddle point, $e_i$. According to Lemma~\ref{lem:sadd} and  \eqref{equ:sadd}, we obtain the closed form solution of \eqref{equ:sadd} as follows:
\begin{align*}
	z_j(t)&=z_j(0)\exp\left(-\mu_j\left(\beta_i-\beta_j\right)t\right)+\sqrt{G_{j,i}}\int_0^t\exp\left( \mu_j\left(\beta_i-\beta_j \right)\left(s-t\right)\right)dB(s)\\
	&=\Big(\underbrace{z_j(0)+\sqrt{G_{j,i}}\int_0^t\exp\big( \mu_j(\beta_i-\beta_j )s\big)dB(s)}_{Q_1}\Big) \underbrace{\exp\left(-\mu_j\left(\beta_i-\beta_j\right)t\right)}_{Q_2}.
	\end{align*}

%By the Ito isometry property of the Ito-Integral, we have
%\begin{align}
%	\EE \left(z_j(t)^2\right)=\left(z_j(0)\right)^2 \exp\left(-2\mu_j\big(\beta_i-\beta_j\big)t\right)+\frac{G_{j,i}}{2 \mu_j\left(\beta_j-\beta_i \right)}\left[ 1-\exp\left(-2\mu_j\big(\beta_i-\beta_j\big)t\right)\right].
%\end{align}
We consider $j=1$. Note at time $t$, $Q_1$ essentially is a random variable with mean $z_1(0)$ and variance $\frac{G_{1,i}\mu_1}{2(\beta_1-\beta_i)}\left(1-\exp\big(-2\mu_1(\beta_1-\beta_i)t\big)\right)$, which has an upper bound $\frac{G_{1,i}\mu_1}{2(\beta_1-\beta_i)}$. $Q_2$, however, amplifies the magnitude of $Q_1$. Then it forces the algorithm escaping from the saddle point $e_i$. We consider the event  $\{w_1(t)^2 > \eta\}$ and a random variable $v(t)\sim N\Big(0,\frac{G_{1,i}\mu_1}{2(\beta_1-\beta_i)}\left(\exp\big(2\mu_1(\beta_1-\beta_i)t\big)-1\right)\Big)$. Because $z_j(0)$ might not be $0$, we have
\begin{align}\label{upper}
	&\PP(w_1(t)^2>\eta) \geq \notag \PP(v^2(t)>1).
	%\frac{1}{\eta^{-1}\epsilon}\left(  \left(z_1(0)\right)^2 \exp(-2\mu_1\left(\beta_i-\beta_1\right)t)+\frac{G_{1,i}}{2 \mu_1\left(\beta_i-\beta_1 \right)}\left[ 1-\exp( -\mu_1\left(\beta_i-\beta_1\right)t)\right] \right).
\end{align} Let the right hand side of \eqref{upper} larger than $95\%$. %, if we let $\frac{G_{1,i}\mu_1}{2(\beta_1-\beta_i)}\left(\exp\big(2\mu_1(\beta_1-\beta_i)t\big)-1\right)=100$. 
Then with a sufficiently small $\eta$, we need
\begin{align}\label{eqn_sde_upper_1}
T_1\asymp \frac{1}{\mu_1(\beta_1-\beta_i) }\log(\frac{200 (\beta_1-\beta_i)}{\mu_1 G_{1,i}}+1)
\end{align}
such that  $\PP(|w^{(\eta)}_1(T_1)|_2^2> \eta)=90\%$.

%Therefore, we can further obtain that 
%\begin{align}
%N_1 = \frac{\eta^{-1}}{2\mu_1(\beta_1-\beta_i) }\log(\frac{200 (\beta_1-\beta_i)}{\mu_1 G_{1,i}}+1).
%\end{align}
%
%
%Let the right hand side of \eqref{upper} be no larger than $\nu$, then with probability at least $1-\nu$, we have $w_1^2 > \tau^2$. Then we have
%
%\begin{align}\label{sde-upper}
%	& \frac{1}{\eta^{-1}\epsilon}\left( \delta^2 e^{-2\mu_{\min}\cdot \sf{gap}\cdot t}+\frac{d\cdot \left( \frac{1}{\mu_1}C_0\cdot C_1+\mu_{\max}C_1\right)}{\sf{gap}}\right)\leq \frac{1}{4}\notag\\
%	\longrightarrow ~& e^{2\mu_{\min}\cdot \sf{gap} \cdot t}\geq \frac{4\cdot \sf{gap}\cdot \delta^2}{\eta^{-1}\epsilon \cdot \sf{gap} -4d\cdot \left( \frac{1}{\mu_1}C_0\cdot C_1+\mu_{\max}C_1\right)}\notag\\
%	\xrightarrow{\textrm{plug $t=\eta N_2$}} ~& N_2\geq \frac{\eta^{-1}}{2\mu_{\min}\cdot \sf{gap} }\cdot\log\left[ \frac{4\cdot \sf{gap} \cdot \delta^2}{\eta^{-1}\epsilon \cdot \sf{gap} -4d\cdot \left( \frac{1}{\mu_1}C_0\cdot C_1+\mu_{\max}C_1\right)}\right].
%\end{align}
%which implies the number of iterations required for the SDE approximation. 
%

Now we consider the time required to converge under the ODE approximation. %the number of iterations required for the ODE approximation. %First we claim that by a specific random initial, that is, $\{X^{(0)}_i\}_{i=1}^d$ (For simplicity, we denote as $X_i$) are independently identically distributed from $N(0,\frac{1}{d}),$ we bound $\frac{w_i^{\mu_1}}{w_1^{\mu_i}}=\cO\left(\left(\sqrt{d}\right)^{\mu_i+\mu_1}\left( \log d \right)^{\mu_1}\right)$ with high probability. 

By Lemma~\ref{lem:ode} with $j=1$, after restarting the counter of time, we have
\begin{align*}
\frac{w_1^{\mu_i}(t)}{w_{i}^{\mu_1}(t)}\geq \eta^{\mu_i/2}\exp(\mu_1\mu_i(\beta_1-\beta_i)t).
\end{align*}
Let the right hand side equal to $1$. Then with a sufficiently small $\eta$ we need 
\begin{align}\label{eqn_ode_1}
T_2 \asymp \frac{\mu_{\max}}{\mu_1\mu_{\min}\cdot \sf{gap}}\log(\eta^{-1}).
\end{align}
such that $\PP\left(\frac{w_i^{(\eta)\mu_1}(T_2)}{w_{1}^{(\eta)\mu_i}(T_2)}\leq 1\right)=\frac{5}{6}$.

Then let $i=1$ in Lemma~\ref{lem:ode}. After restarting the counter of time, we have
\begin{align*}
& \frac{w_i^{\mu_1}(t)}{w_{1}^{\mu_i}(t)}\leq C \exp(\mu_{\max})\exp(\mu_1\mu_i(\beta_i-\beta_1)t)\\
\Longrightarrow & w_i^2\leq  \left(C\exp(\mu_{\max})\exp(\mu_1\mu_i(\beta_i-\beta_1)t)\right)^{2/\mu_1}
\end{align*}
where $\exp(\mu_{\max})$ comes from the above stage and $C$ is a constant containing $G_{1,i}$ and $G_{i,j}$. The second inequality holds due to the fact that $w_1\leq 1$, mentioned in the proof of Lemma~\ref{lem:ode}.
Therefore, given $\sum_{i=2}^d w_i^2\leq \kappa \eta^{1+2\delta}$ and a sufficiently small $\eta$, we need 
\begin{align}\label{eqn_ode_2}
T_2'\asymp \frac{\mu_{\max}}{\mu_1\mu_{\min}\cdot \sf{gap}}\log (\eta^{-1})
\end{align}
such that $\PP\left(\frac{|w^{(\eta)}_1(T_2')|^2}{\norm{w^{(\eta)}(T_2')}_2^2}>1-\kappa \eta^{1+2\delta}\right)=\frac{8}{9}$.

Then the algorithm goes into Phase III. According to Lemma~\ref{lem:sde} and \eqref{equ:sde}, we obtain the closed form solution of \eqref{equ:sde} as follows:
\begin{align*}
	z_i(t)=z_i(0)\exp\left(-\mu_i\left(\beta_1-\beta_i\right)t\right)+\sqrt{G_{i,1}}\int_0^t\exp\left( \mu_i\left(\beta_1-\beta_i \right)\left(s-t\right)\right)dB(s).
\end{align*}
By the Ito isometry property of the Ito-Integral, we have
\begin{align}
	\EE \left(z_i(t)\right)^2=\left(z_i(0)\right)^2 e^{-2\mu_i\left(\beta_1-\beta_i\right)t}+\frac{G_{i,1}}{2 \mu_i\left(\beta_1-\beta_i \right)}\left[ 1-e^{ -2\mu_i\left(\beta_1-\beta_i\right)t}\right].
\end{align}
Then we consider the complement of the event $\{w_1^2 > 1-\epsilon\}$. By Markov inequality, we have 
\begin{align}\label{upper}
	& \PP(w_1^2\leq1-\epsilon) \notag\\
	=&\PP \left(\sum_{i=2}^d w_i^2\geq \epsilon \right)\leq \frac{\EE\left( \sum_{i=2}^d w_i^2\right)}{\epsilon}=\frac{\EE \left(\sum_{i=2}^d z_i^2\right)}{\eta^{-1}\epsilon}\notag \\
        =& \frac{1}{\eta^{-1}\epsilon}\left(\sum_{i=2}^d  \left(z_i(0)\right)^2 e^{-2\mu_i\left(\beta_1-\beta_i\right)t}+\frac{G_i}{2 \mu_i\left(\beta_1-\beta_i \right)}\left[ 1-e^{ -2\mu_i\left(\beta_1-\beta_i\right)t}\right] \right)\notag\\
      \leq & \frac{1}{\eta^{-1}\epsilon}\left( \eta^{-1}\delta^2 e^{-2\mu_{\min}\cdot \sf{gap} \cdot t}+\frac{\phi}{2\mu_{\min}\cdot\sf{gap}}\right).
\end{align}
Let the right hand side of \eqref{upper} be no larger than $\frac{1}{16}$. 
\begin{align}%\label{sde-upper}
	& \frac{1}{\eta^{-1}\epsilon}\left( \eta^{-1}\delta^2 e^{-2\mu_{\min}\cdot \textrm{gap}\cdot t}+\frac{\phi}{2\mu_{\min}\cdot\sf{gap}}\right)\leq \frac{1}{16}\notag\\
	\Longrightarrow ~& e^{2\mu_{\min}\cdot \textrm{gap} \cdot t}\geq \frac{16 \cdot\mu_{\min}\cdot \sf{gap}\cdot \delta^2}{\epsilon\cdot \mu_{\min} \cdot \textrm{gap} -16\cdot \eta\cdot \phi}\notag.
\end{align}

Then after restarting the counter of time, we need 
\begin{align}\label{sde-upper}
 T_3\asymp\frac{1}{\mu_{\min}\cdot \sf{gap} }\cdot\log\left( \frac{\mu_{\min}\cdot \sf{gap}\cdot \delta^2}{\epsilon\cdot \mu_{\min} \cdot \textrm{gap} -16\cdot \eta\cdot \phi}\right).
	%\xrightarrow{\textrm{plug $t=\eta N_3$}} ~& N_3\geq \frac{\eta^{-1}}{2\mu_{\min}\cdot \sf{gap} }\cdot\log\left[ \frac{4\cdot \sf{gap} \cdot \delta^2}{\eta^{-1}\epsilon \cdot \textrm{gap} -4d \cdot \left( \frac{1}{\mu_1}C_0\cdot C_1+\mu_{\max}C_1\right)}\right].
\end{align}
such that  $\PP(w_1^2(T_3) \geq 1-\epsilon)\geq \frac{15}{16}$.

Combining \eqref{eqn_sde_upper_1}, \eqref{eqn_ode_1}, \eqref{eqn_ode_2}, \eqref{sde-upper}, if our algorithm start from a saddle, then with probability at least $\frac{5}{8}$, we need
\begin{align}\label{T}
T=T_1+T_2+T_2'+T_3\asymp \frac{\mu_{\max}/\mu_{\min}}{\mu_1\cdot\sf{gap}}\log\left(\eta^{-1}\right)
\end{align}
such that $w_1^2(T) > 1-\epsilon$. 
Moreover, we choose 
\begin{align}\label{choose}
	\eta \asymp \frac{\epsilon \cdot \mu_{\min}\cdot \sf{gap} }{\phi}.
\end{align}
Combining \eqref{T} and \eqref{choose} together, we get the asymptotic sample complexity
\begin{align}\label{choosing}
	N\asymp \frac{T}{\eta}\asymp  \frac{\phi\cdot \mu_{\max}/\mu_{\min}}{\epsilon\cdot\mu_1\cdot\mu_{\min}\cdot\sf{gap}^2}\log\left(\frac{\phi}{\epsilon\cdot\mu_{\min}\cdot \sf{gap}}\right)
	%\left( \frac{d\left( \frac{1}{\mu_1}+\mu_{\max}\right)}{\epsilon\cdot \sf{gap}^2\cdot \mu_{\min}}\log\frac{d^{\frac{\mu_{\max}}{\mu_1}+1}}{\epsilon \cdot  \sf{gap} }\right).
\end{align}
such that with probability at least $\frac{5}{8}$, we have $\norm{\hat{W}-W^*}_2^2\leq \epsilon$.
\end{proof}